\documentclass[journal]{IEEEtran}

\usepackage{amsmath}
\usepackage{amsthm}
\usepackage{amssymb}
\usepackage{tikz}
\usepackage{pgfplots}
\usepackage{xspace}

\usepackage{ifpdf}
\usepackage{cite}
\usepackage{amsmath}

\usepackage{algorithmic}

\usepackage{array}
\newcolumntype{H}{>{\setbox0=\hbox\bgroup}c<{\egroup}@{}}

\ifCLASSINFOpdf
\ifCLASSINFOpdf
\else
\fi
\usepackage{array}

\usepackage{url}

\usepackage{pgfplots}
\pgfplotsset{plot coordinates/math parser=false}
\pgfplotsset{soldot/.style={color=blue,only marks,mark=*}} 
\pgfplotsset{holdot/.style={color=blue,fill=white,only marks,mark=*}}
\definecolor{dark-gray}{gray}{0.3}
\usepackage{graphicx} 
\usepackage{color}
\usepackage{thmtools,thm-restate}
\usepackage{url}

\usepackage{amsmath}
  \usepackage{algorithm}
\usepackage{algorithmic}
\usepackage{caption}
\usepackage{xspace}
\usepackage{multirow}

\newcommand{\prob}[1]{\ensuremath{\text{Pr}\left\{#1 \right\} }}
\newcommand{\fhm}{fast hypermutation\xspace}

\newcommand{\foea}{Fast $(1 + 1)$~EA$_\beta$\xspace}
\newcommand{\fastia}{Fast~$(1+1)$~IA\xspace}
\newcommand{\fastianp}{Fast~$(1+1)$~IA$_{\gamma}$\xspace}

\newcommand{\fastoptia}{Fast~Opt-IA$_\gamma$\xspace}
\newcommand{\oneoneea}{$(1+1)$~EA}

\newcommand{\oneoneiahype}{$(1+1)$~IA\xspace}

\newcommand\T{\rule{0pt}{2.6ex}}      
\newcommand\B{\rule[-1.2ex]{0pt}{0pt}}

\newtheorem{theorem}{Theorem}
\newtheorem{definition}{Definition}
\newtheorem{lemma}[theorem]{Lemma}
\newtheorem{corollary}[theorem]{Corollary}

\newcommand{\partition}{\textsc{Partition}\xspace}

\newcommand{\jumpf}{\textsc{Jump}\xspace}
\newcommand{\clifff}{\textsc{Cliff}\xspace}

\newcommand{\hypfcm}{FCM$_\gamma$\xspace}

\newcommand{\hypheavy}{FCM$_\beta$\xspace}

\begin{document}

\title{Fast  Immune System Inspired Hypermutation Operators for Combinatorial Optimisation}


\author{\IEEEauthorblockN{Dogan Corus,
Pietro S. Oliveto, and
Donya Yazdani}\\
\IEEEauthorblockA{Department of Computer Science,
University of Sheffield, S1 4DP, UK}\\
\IEEEauthorblockN{d.corus@sheffield.ac.uk, p.oliveto@sheffield.ac.uk, dyazdani1@sheffield.ac.uk}
}

%
%
%
%



\IEEEtitleabstractindextext{%
\begin{abstract}
Various studies have shown that immune system inspired  
hypermutation operators can allow artificial immune systems (AIS) to  be very efficient at escaping local optima of multimodal optimisation problems.
However, this efficiency comes at the expense of considerably slower runtimes during the exploitation phase compared to standard evolutionary algorithms. We propose modifications to the traditional `hypermutations with mutation potential' (HMP) that allow them to be efficient at exploitation as well 
as maintaining their effective explorative characteristics.  Rather than deterministically evaluating fitness after each bit-flip of a hypermutation,  we sample the fitness function stochastically with a `parabolic' distribution which allows  the `stop at first constructive mutation' (FCM) variant of HMP to reduce the linear amount of wasted function evaluations when no improvement is found to a constant.  
The stochastic distribution also allows the removal of the FCM mechanism altogether as originally desired in the design of the HMP operators.
We rigorously prove the effectiveness of the proposed operators 
for all the benchmark functions where the performance of HMP is rigorously understood in the literature and validating the gained insights to show linear speed-ups for the identification of high quality approximate solutions to classical NP-Hard problems from combinatorial optimisation. We then show the superiority of the HMP operators to the traditional ones in an analysis of the complete 
 standard Opt-IA AIS, where the stochastic evaluation scheme allows HMP and ageing operators to work in harmony.
Through a comparative performance study of other `fast mutation' operators from the literature, we conclude that a power-law distribution for the parabolic evaluation scheme is the best compromise in black box scenarios where little problem knowledge is available.
\end{abstract}

}

\maketitle

\IEEEdisplaynontitleabstractindextext

%
\IEEEpeerreviewmaketitle

\section{Introduction}

%
%
%
%
Several artificial immune systems (AISs) inspired by Burnet's clonal selection principle~\cite{Burnet1959} have been developed to solve optimisation problems.
Amongst these, Clonalg \cite{DecastroVonzuben2002}, the B-Cell algorithm \cite{KelseyTimmis2003} and Opt-IA \cite{CutelloTEVC} are the most popular.
Being inspired by the immune system, a common feature of these algorithms is that they have particularly high mutation rates compared to more traditional evolutionary algorithms (EAs) which, inspired in turn by natural evolution, have traditionally used considerably lower mutation rates.

For instance, the {\it contiguous somatic hypermutations} (CHM) used by the B-Cell algorithm, choose two random positions in the genotype of a candidate solution and 
flip all the bits in between\footnote{A parameter may be used to define the probability that each bit in the region actually flips.}. This operation results in a linear number of bits 
being flipped on average in a mutation. 
The {\it hypermutations with mutation potential} (HMP) used by Opt-IA also flip a linear number of bits. 
However, it has been proved that their basic originally proposed static version, where a linear number of bits are always flipped, cannot  optimise efficiently any function with any polynomial number of optima~\cite{CorusOlivetoYazdani2019TCS}. 
On the other hand, much better performance has been shown in theory~\cite{CorusOlivetoYazdani2019TCS} and in practice~\cite{CutelloTEVC} for the version that evaluates the fitness after each bit flip in the hypermutation and stops the process if an improving solution is found (i.e., static HMP with {\it stop at first constructive mutation} (FCM)). 

Various studies have shown how these high mutation rates allow AISs to escape from local optima for which more traditional randomised search heuristics struggle.
Jansen and Zarges proved for a benchmark function called Concatenated Leading Ones Blocks (CLOB) an expected runtime of $O(n^2 \log n)$ using CHM versus the exponential time required by EAs relying on standard bit mutations (SBM) since many bits need to be flipped simultaneously to make progress~\cite{JansenZargesTCS2011}. 
Similar effects have also been shown for instances of the longest common subsequence \cite{JansenZarges2012} and vertex cover \cite{JansenOlivetoZarges2011vertex} combinatorial optimisation problems with practical applications, where CHM efficiently escapes local optima while EAs (with and without crossover) are trapped for exponential time. Also, the HMP with FCM of Opt-IA have been proven to be considerably efficient at escaping local optima such as those of the multimodal \textsc{Jump}, \textsc{Cliff}, and \textsc{Trap} benchmark functions that 
standard EAs find very difficult \cite{CorusOlivetoYazdani2019TCS}. Furthermore, their effectiveness at escaping from local optima  has been shown to guarantee arbitrarily good constant approximations for the NP-Hard \partition problem while RLS and EAs may get stuck on bad approximations~\cite{CorusOlivetoYazdaniAIJ2019}.

The efficiency on multimodal problems of these AISs comes at the expense of being 
considerably slower than EAs in the final exploitation phase of the optimisation process 
when few bits have to be flipped. For instance, CHM requires  $\Theta(n^2 \log 
n)$ expected function evaluations to optimise the easy \textsc{OneMax} and 
\textsc{LeadingOnes} unimodal benchmark functions. Indeed, it has recently  been shown that CHM 
requires at least $\Omega(n^2)$ function evaluations to optimise any function 
since its expected runtime for its easiest function is 
$\Theta(n^2)$~\cite{EasiestFunctions}. Another disadvantage of CHM is that it is {\it biased}, in the sense that it behaves differently according to the order in which the information is encoded in the bit-string.
In this sense, the {\it unbiased} HMP operators used by Opt-IA are easier and more convenient to apply as their performance does not depend on the encoding order of the bit positions.
However, the static HMP operator with FCM has also been proven to have runtimes of respectively $\Theta(n^2 \log n)$ expected fitness evaluations for \textsc{OneMax} and $\Theta(n^3)$ for \textsc{LeadingOnes}. Recently, speed-ups in the exploitation phase have been shown for the Inversely Proportional HMP variant (INV HMP), that aims to decrease the mutation rate as the local and global optima are approached~\cite{CorusOlivetoYazdaniIPH2019}. On one hand, while faster, INV HMP operators are still asymptotically slower than RLS and EAs for easy hillclimbing problems such as \textsc{OneMax} and \textsc{LeadingOnes}. On the other hand, the speed-ups at hill-climbing are achieved at the expense of losing their power at escaping from local optima via mutation. Since the mutation rates are lowest on local optima, it is unlikely that the INV HMP operator can escape quickly via hypermutation.
%


 In this paper, we propose a modification to the static HMP operator to allow it to be very efficient in the exploitation phases while  
maintaining its essential characteristics for escaping from local optima.
Rather than evaluating the fitness after each bit flip of a hypermutation as the traditional HMP with FCM requires, we propose to evaluate the fitness based on the probability that the mutation will be successful.

The probability of hitting a specific point at Hamming distance $i$ from the current point (i.e., ${n \choose i}^{-1}$) decreases exponentially with the Hamming distance for $i < n/2$ and then it increases again in the same fashion. Based on this observation, we evaluate each bit following a  parabolic distribution such that the probability of evaluating the $i$th bit flip decreases as $i$ approaches $n/2$ and then increases again. 
We call the resulting operator \hypfcm{} and embed it in an algorithm called \fastianp.

We rigorously prove that the \fastianp locates local optima asymptotically as fast as random local search (RLS) for any function where the expected runtime of RLS can be proven using the standard artificial fitness levels method (AFL). At the same time, the operator is still exponentially faster than EAs for the standard multimodal 
\textsc{Jump}, \textsc{Cliff}, and \textsc{Trap} benchmark functions. 

We also validate the insights gained from the analysis for benchmark functions on classical NP-Hard problems from combinatorial optimisation. We first derive a smaller upper bound compared to static HMP on the expected runtime required by the \fastianp to find arbitrarily good constant approximations to the \partition problem. This result is surprising because the proof requires mutations of approximately $n/2$ bits. This is exactly the range of mutations which is penalised by our proposed distribution. Nevertheless, the greater exploitative capabilities of the hypermutation operator lead to a linear factor smaller upper bound on the expected runtime because the time spent in the hillclimbing phases dominates the overall expected runtime. Thus, the utility of our modifications is proven on a problem with many real world applications. 
Recall that EAs using SBM may get stuck on bad $4/3$ approximations for exponential  time.
We also rigorously prove linear speed-ups for the NP-Hard \textsc{Vertex Cover} problem, compared to the static HMP operator. We show these both for identifying feasible solutions if a node representation is used for the bit-string, and to identify 2-approximations if an edge based representation is used. 

We then evaluate the performance of the {\it fast} hypermutation operator using the parabolic evaluation distribution in the context of complete AISs.
Indeed hypermutations with mutation potential are usually applied in conjunction with ageing operators in the standard Opt-IA AIS~\cite{CutelloTEVC}.
The power of ageing at escaping local optima has recently been enhanced by showing how, by accepting inferior solutions when stuck on local optima, it makes the difference between polynomial and exponential runtimes for the \textsc{Balance} function 
from dynamic optimisation~\cite{OlivetoSudholt2014}. For very difficult instances of \textsc{Cliff}, where standard RLS and elitist EAs require exponential time, ageing even makes RLS asymptotically as fast as any unbiased mutation based algorithm can be on any function with unique optimum~\cite{Lehre2012} i.e., by running in $O(n \ln n)$ expected time~\cite{CorusOlivetoYazdani2019TCS}.

However, the power of ageing at escaping local optima is lost when it is used in combination with static HMP.
In particular, the  FCM mechanism does not allow the operator to return solutions of lower quality apart from the complementary bit-string, thus cancelling the advantages of ageing. Furthermore, the
high mutation rates combined with FCM make the algorithm return to the previous local optimum with very high probability. 
We show how these problems are naturally solved by our newly proposed operators that do not evaluate all bit flips in a hypermutation. 
We rigorously prove that the resulting algorithm, called Opt-IA$_\gamma$, benefits from the modified operator showing that it allows the ageing operator to escape from local optima by accepting the lower quality solutions returned by the 
\hypfcm operator when it does not find improvements. 
However, to achieve this behaviour the evaluation probabilities after each bit flip have to be set to prohibitively low values such that the applied operator effectively does not mutate many bits anymore (i.e. it does not hypermutate; similarly to the INV HMP of~\cite{CorusOlivetoYazdaniIPH2019} when it is located on the best found local optimum).

To address this problem, and to further evaluate the general performance of the proposed {\it fast} HMP operator, we perform a comparative analysis with other '{\it fast} mutation' operators that have recently appeared in the evolutionary computation literature\cite{Doerretal2017,FGQWPPSN18,FQWGECCO18}. The analysis leads to the conclusion that a parabolic power-law distribution is the best compromise for the {\it fast} hypermutation operator in black box scenarios where limited problem knowledge is available. Such a distribution allows a greater balance between large and small mutations. Hence, local optima may be escaped from, by performing large or small mutations to new basins of attraction that are either of better or of worse quality (i.e., due to ageing). We show that the obtained AISs perform asymptotically either at least as well, or better, than all the considered  algorithms over the large range of unimodal and multimodal problems considered in this paper.
Due to page restrictions the proofs of the theorems are presented as supplementary material as well as a self-contained version of the paper.



\section{AISs with Probabilistic Sampling Distributions}

Hypermutations with mutation potential (HMP) differ from the standard bit mutations (SBM) used traditionally in evolutionary computation by flipping a linear number of distinct bits $M=cn$ for a constant $0<c \leq 1$.
It has been shown that in their basic static version, where they only evaluate the result of the $M$ bit flips, they are inefficient at optimising any function with up to a polynomial number of optima~\cite{CorusOlivetoYazdani2019TCS}. 
In the {\it stop at the first  constructive mutation} (FCM) variant they mutate {\it at most} $M=cn$ distinct bits 
(i.e., for this reason $M$ is called the mutation {\it potential}). 
After each of the M bit-flips, they  evaluate the fitness of the constructed solution. If an improvement over the original solution is found before the $M$th bit-flip, then the operator stops and returns the improved solution~\cite{CutelloTEVC}. 
This behaviour prevents the hypermutation operator to waste further fitness function 
evaluations if an improvement has already been found. However, for any realistic objective function, the number of 
iterations where there is an improvement constitutes an asymptotically small 
fraction of the total runtime. Hence, the fitness function evaluations saved due to 
the FCM stopping the hypermutation have a very small impact on the global performance 
of the algorithm. 
While they have been shown to be more efficient than SBM to escape from local optima, this performance comes at the expense of being up to a linear factor slower at hillclimbing in the exploitative phases of the optimisation process~\cite{CorusOlivetoYazdani2019TCS}. 

Therefore, we propose an alternative HMP operator using FCM, called \hypfcm for simplicity, that only evaluates the fitness after each bit-flip with some probability. 
Since setting the HMP parameter to $c=1$ (i.e., $M=n$) allows the operator to reach any point in the search space with positive probability, we will only consider this parameter setting throughout the paper as was also done in previous theoretical analyses~\cite{CorusOlivetoYazdani2017,CorusOlivetoYazdani2019TCS}.

We propose the use of the following parabolic probability distribution depicted in Figure~\ref{fig:one}. Let $p_i$ be the probability that the solution is evaluated after the $i$th bit has been flipped.
Then,  
\begin{align}
\label{probGamma}
p_i= 
\begin{cases}
		1/e & \text{for}\; i=1 \;\text{and}\; i=n,\\   
        \gamma/i & \text{for}\; 1<i\leq n/2,\\
        \gamma/(n-i) & \text{for}\; n/2<i<n.\\
\end{cases}
\end{align}
where the parameter $\gamma$ should be in~$0< \gamma \leq 1$ (however, any $0< \gamma < 1/e$ is an efficient choice for the results that we will present).

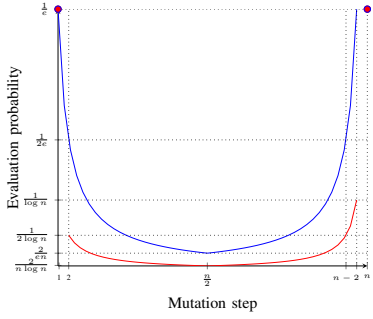
\begin{figure}[t]

\centering
\begin{tikzpicture}[ scale=0.6]
 \pgfmathsetmacro\ytkn{30} 
 \pgfmathsetmacro\ytknhalf{\ytkn/2} 
 \pgfmathsetmacro\ytkni{.22/\ytkn}
 \pgfmathsetmacro\ytknii{2/((2.7)*\ytkn)}
 \pgfmathsetmacro\ytkna{\ytkn-1} 
\pgfmathsetmacro\ytknb{\ytkn-2} 
 \pgfmathsetmacro\ytka{1/((2.7))}
 \pgfmathsetmacro\ytkb{1/(2*(2.7))}
 \pgfmathsetmacro\ytkc{1/(20)}
 \pgfmathsetmacro\ytkd{1/(10)}

\begin{axis}[
    axis x line=center, 
    axis y line=middle, 
      ytick={\ytkni, \ytknii, \ytka,\ytkb,  \ytkc, \ytkd},
yticklabels={$\frac{2}{n\log{n}}$, 
$\frac{2}{en}$,$\frac{1}{e}$,$\frac{1}{2e}$,$\frac{1}{2\log{ n } } $ , 
$\frac{1}{\log{n}}$, $$ ,$$ },
      xtick={1.1, 2,   \ytknb,\ytkna, \ytkn, \ytknhalf},
    xticklabels={1,2,$n-2$,$ $,$n$,$\frac{n}{2}$ },
    domain=1:\ytkn,
    ticklabel style = {font=\tiny},
      x label style={at={(axis description cs:0.5,-0.1)},anchor=north},
    y label style={at={(axis description cs:-0.1,.5)},rotate=90,anchor=south},
     xlabel={Mutation step},
    ylabel={Evaluation probability}]
]

\addplot[domain=1:\ytkn/2, blue] {1/((2.7)*x)};
\addplot[domain=\ytkn/2:\ytkn-1,blue] {1/(2.7*(\ytkn-x))};
\addplot[holdot, color=red] 
coordinates{(1,\ytka)(\ytkn,\ytka)};
\draw[dotted] (axis cs:1,\ytka) -- (axis cs:\ytkn,\ytka);
\draw[dotted] (axis cs:1,\ytkb) -- (axis cs:\ytkn,\ytkb);
\draw[dotted] (axis cs:1,\ytkc) -- (axis cs:\ytkn,\ytkc);
\draw[dotted] (axis cs:1,\ytknii) -- (axis cs:\ytkn,\ytknii);
\draw[dotted] (axis cs:1,\ytkd) -- (axis cs:\ytkn,\ytkd);
\draw[dotted] (axis cs:2,0) -- (axis cs:2,\ytka);
\draw[dotted] (axis cs:\ytknb,0) -- (axis cs:\ytknb,\ytka);
\draw[dotted] (axis cs:\ytkna,0) -- (axis cs:\ytkna,\ytka);
\draw[dotted] (axis cs:\ytkn,0) -- (axis cs:\ytkn,\ytka);
\addplot[holdot, color=blue,fill=red, ] 
coordinates{(1,\ytka  )(\ytkn,\ytka )};
\addplot[domain=2:\ytkn/2, red] {1/(10*x)};
\addplot[domain=\ytkn/2:\ytkn-1,red] {1/(10*(\ytkn-x))};

\end{axis}
\end{tikzpicture}
 \caption{The parabolic evaluation probabilities (\ref{probGamma}) for 
{\color{red}$\gamma=1/\log{n}$} and  {\color{blue}$\gamma=1/e$}.} \label{fig:one}
\end{figure}
%

The lower the value of $\gamma$, the fewer the expected fitness function evaluations that occur in each hypermutation. 
In particular, with a sufficiently small value for $\gamma$, the number of wasted evaluations can be dropped to the 
order of $O(1)$ per iteration instead of the linear amount wasted by the traditional operator when improvements are not found. At the same time, it still flips many bits (i.e., it hypermutates) as desired. 
The resulting hypermutation operator is formally defined as follows.

\begin{definition}[\hypfcm{}]\label{def:hyp-fcm}
The \hypfcm{} operator flips at most $n$ distinct bits selected uniformly at random. It evaluates the fitness after the \textbf{{\it i}}th bit-flip with 
probability $p_i$ (as defined in (\ref{probGamma})) and remembers the last evaluation. \hypfcm{}  stops flipping bits when it finds an improvement; if no improvement is found, it will return the last evaluated solution. If no evaluations are made, the parent will be returned.
\end{definition}

\begin{algorithm}[t]
\caption{\fastianp{} for maximisation}
\begin{algorithmic}[1]
\STATE{Initialise $x$ u.a.r (uniformly at random).}
\WHILE{the termination criterion is not met}
\STATE{create offspring $y$ using \hypfcm{};} 
\STATE{{\bf if} $f(y) \geq f(x)$, {\bf then} $x:=y$;}
\ENDWHILE
\end{algorithmic}
\label{alg:fastia}
\end{algorithm}

In the next section, we will prove the benefits of \hypfcm{} over the standard HMP with FCM, when incorporated into a $(1+1)$ framework. 
We will refer to the algorithm as \fastianp to distinguish it from the standard \oneoneiahype which uses the traditional HMP operator i.e., that evaluates the fitness of the constructed solutions deterministically after each bit-flip of the hypermutation. 
Similar benefits may also be shown for population-based AISs but we will refrain to do so since populations do not lead to improved performance for the considered benchmark problems.
The \fastianp is formally defined in Algorithm \ref{alg:fastia}.
It keeps  a single individual in the population and uses \hypfcm{} to perturb it in every 
iteration. If the offspring is not worse than its parent, then it replaces the parent for the next iteration; otherwise the parent is kept.

Traditional static FCM operators are not suited to be used in conjunction with ageing operators if the power of the latter at escaping local optima is to be exploited~\cite{CorusOlivetoYazdani2019TCS}. While ageing operators allow to exploit solutions of lower quality to escape from local optima, the traditional HMP with FCM returns a solution if it is an improvement or it always returns the complementary bit string (which is unlikely to be useful very often).
 However, this is not true for the above defined \hypfcm{} variant. If no improvements are found, FCM$_\gamma$ returns the last evaluated solution, which is not necessarily the complementary bit string. Hence, the above operator has higher chances of being effective at escaping from local optima than traditional HMP with FCM by identifying a variety of new, potentially promising, basins of attraction.  
For sufficiently small values of the parameter $\gamma$, only one function evaluation per hypermutation is performed in expectation 
(although all bits will be flipped i.e., it hypermutates).
Since \hypfcm returns the last evaluated one, 
this solution will be returned by the operator as it is the only one it has encountered.
Interestingly, this behaviour is similar to that of the traditional HMP operator without FCM that also evaluates one point per hypermutation and returns it.
However, while the traditional version has been to proven to have exponential expected runtime for any function with any polynomial number of optima~\cite{CorusOlivetoYazdani2019TCS}, we will show in the following sections that the {\it fast} HMP 
can be very efficient. From this point of view, with appropriate parameter settings, 
\hypfcm is a very effective way to perform hypermutations with mutation potential without FCM as originally desired~\cite{CutelloTEVC}.

We will analyse the \hypfcm operator 
in a complete Opt-IA that uses cloning, hypermutation and ageing.
The modified Opt-IA algorithm using \hypfcm, which we call \fastoptia, is depicted in Algorithm 
\ref{alg:fastoptia}. 
We will use the  \textit{hybrid ageing} operator as in \cite{CorusOlivetoYazdani2019TCS,OlivetoSudholt2014}, which allows the algorithm to 
escape from local optima. Hybrid ageing removes candidate solutions (i.e., b-cells) 
with probability $p_{die}$ once they have reached an age threshold 
$\tau$. 
After initialising a population of $\mu$ solutions with $age=0$,  the algorithm creates $dup$ copies of each solution in each iteration. 
These copies are all mutated by the hypermutation operator, creating a population of mutants called $P^{(hyp)}$. These mutants inherit the age of their parents if they do not improve the fitness; otherwise their age is set to zero. At the next step, all solutions with $age \geq \tau$ will be removed with probability $p_{die}$. 
If fewer than $\mu$ individuals have survived ageing, then the population is filled up with new randomly generated individuals.
At the selection for replacement phase, the best $\mu$ solutions are chosen to form the population for the next generation.
In Section \ref{sec:fastoptia}, 
we will prove the benefits of the \fastoptia for all the  unimodal and multimodal benchmark functions for which the performance of the Opt-IA with traditional static HMP has been proven in the literature.

As usual in evolutionary computation we will evaluate the performance  of the algorithms by calculating the expected number of fitness function evaluations until the optimum (or an approximation for the NP-Hard problems) is identified (i.e. expected runtime). Hence, we do not specify any termination criterion for the evolutionary loops of the algorithms.

 \subsection{Mathematical Tools for the Analysis}

 In this section, we introduce the mathematical tools from the literature which we will use to carry out our analysis.
 
 We will apply the following theorem by Serfling which provides an 
 upper bound on the probability that the outcome of a hypergeometrically distributed random variable exceeds a given value. While the more common Chernoff bounds could also be used to obtain the same results, we prefer to use Serfling's theorem because the hypergeometric distribution better represents the behaviour of the considered hypermutation operators on functions of unitation (i.e., functions where the output depends exclusively on the number of 1-bits in the bit-string).
 \begin{theorem}[Serfling \cite{serfling1974probability}]\label{thm:serfling}
  Consider a set 
 $C:=\{c_1, \ldots, c_n\}$ consisting of $n$ elements, 
 with $c_i \in R$ where $c_{min}$ and $c_{max}$ are the smallest and largest 
 elements in $C$ respectively. Let $\bar{\mu}:= (1/n)\sum_{j=1}^{n}c_i$, be the 
 mean of $C$. Let $1\leq i \leq k \leq n$ and $X_i$ denote the $i$th 
 draw without replacement from $C$ and $\bar{X}:=(1/k)\sum_{j=1}^{k} X_i$  the 
 sample mean. For $1\leq k \leq n$, and $\lambda>0$
  \[Pr\left\{  \sqrt{k}  (\bar{X} - \bar{\mu}) \geq \lambda \right\}\leq 
 \exp\left(-\frac{2\lambda^2}{(1-f^{*}_{k})(c_{max}-c_{min})^{2}}\right)\] 
 where $f^{*}_{k}:= \frac{k-1}{n}$. 
 \end{theorem}

 Another tool from the literature which is widely used in the analysis of HMP operators is the following Ballot theorem. It was first applied by Jansen and Zarges in~\cite{JansenZarges2011} to bound the expected runtime of inversely proportional HMP. 
 \begin{theorem}[Ballot Theorem~\cite{feller1968}] \label{thm:ballot} 
 Suppose that, in a ballot, candidate P scores $p$ votes and candidate Q scores $q$ votes, where $p>q$. The probability that throughout the counting there are always more votes for P than for Q equals $(p-q)/(p+q)$. 
 \end{theorem}
 
 Artificial Fitness Levels (AFL) is a standard technique used in the theory of evolutionary computation to derive upper bounds on the expected 
 runtime of $(1+1)$ evolutionary algorithms~\cite{OlivetoBookChapter,jansenbook,OlivetoLehrBookChapter2019}. AFL divides the search space into $m$ 
 mutually exclusive partitions $A_1 \cdots, A_m$ such that all the points in 
 $A_i$ have smaller fitness than any point which belong to $A_j$ for all $j>i$. The 
 last partition, $A_m$ only includes the global optimum. If $p_i$ is the 
 smallest probability that an individual belonging to $A_i$ mutates to an 
 individual belonging to $A_j$ such that $i<j$, then the expected time to find 
 the optimum is $E(T) \leq \sum_{i=1}^{m-1}1/p_i$. 
 We will show in the following section that the results obtained by using AFL to derive upper bounds on the expected runtime of simpler randomised local search heuristics can be easily converted into upper bounds on the expected runtime of the \fastianp.

 Finally, we will apply the standard multiplicative drift theorem which is widely used in the runtime analysis of stochastic search heuristics.

 \begin{theorem}[Multiplicative Drift Theorem~\cite{DoerrMulti,lenglerNew,OlivetoLehrBookChapter2019}]{\label{th:multidrift}} 
 Let $\{X_t\}_{t \geq}$ be a sequence of random values taking the values in some set $S$. Let $g:S \rightarrow\{0\} \cup \mathbb{R}_{\geq 1}$ and assume that $g_{max}:=\max\{g(x)\mid x \in S\}$ exists. Let $T:=min\{t\geq0:g(x_t)=0\}$. If there exists $\delta>0$ such that
 $E(g(X_{t+1})\mid g(X_t)) <(1-\delta)g(X_t)$,
 then
 $E(T)\leq \frac{1}{\delta}(1+ \ln g_{max})$
 and for every $c>0$, 
 $Pr(T>\frac{1}{\delta}(\ln g_{max}+c)) \leq e^{-c}.$
 
 \end{theorem}

\begin{algorithm}[t]
\caption{\fastoptia{} for maximisation}
\begin{algorithmic}[1]
\STATE{Initialise $P:=\{x_1,...,x_\mu\}$, a population of $\mu$ solutions u.a.r and set $x_i.age:=0$ for $i:=\{1,...\mu\}$};
\WHILE{the termination criterion is not met}
\FOR{all $x \in P$}
\STATE{set $x.{age}:=x.{age}+1$;}
\STATE{copy $x_i$ $dup$ times and add the copies to $P^{(clo)}$};
\ENDFOR
\FOR{all $x \in P^{(clo)}$}
\STATE{create $y$ using \hypfcm{};}
\STATE{ {\bf if} $f(y)>f(x)$, {\bf then} $y.{age}:= 0$};
\STATE{{\bf else} $y.{age}:= x.age$;}
\STATE{add $y$ to $P^{(hyp)}$;}
\ENDFOR

\STATE{add $P^{(hyp)}$ to $P$, set $P^{(hyp)}:=\emptyset$;}
\STATE{with probability {\color{black}$p_{die}:=1-\frac{1}{(dup+1)\cdot\mu}$}, remove any $x_i \in$ 
$P$ with $x_i.{age} \geq \tau$;} 
\STATE{{\bf if} $|P| < \mu $, {\bf then} add $\mu-|P|$ solutions to $P$ with $age:=0$ generated u.a.r;}
\STATE{{\bf else if} $|P| > \mu $, {\bf then} remove $|P|-\mu$ solutions with the lowest fitness from $P$ {\color{black}breaking ties 
u.a.r};}
\ENDWHILE
\end{algorithmic}
\label{alg:fastoptia}
\end{algorithm}

\section{Artificial Fitness Levels for Fast Hypermutations}
In~\cite{CorusOlivetoYazdani2019TCS},  a mathematical methodology was devised that allows to convert upper bounds on the expected runtime
of randomised local search (RLS) into valid upper bounds on the expected runtime of the traditional static HMP operators. 
In this section, we will extend such methodology
so that it can also be applied to the {\it fast} HMP operator introduced in this paper.

Artificial Fitness Levels (AFL) is a standard technique used in the theory of evolutionary computation to derive upper bounds on the expected 
runtime of $(1+1)$ evolutionary algorithms~\cite{OlivetoBookChapter,jansenbook,OlivetoLehrBookChapter2019}. AFL divides the search space into $m$ 
mutually exclusive partitions $A_1 \cdots, A_m$ such that all the points in 
$A_i$ have smaller fitness than any point which belong to $A_j$ for all $j>i$. The 
last partition, $A_m$ only includes the global optimum. If $p_i$ is the 
smallest probability that an individual belonging to $A_i$ mutates to an 
individual belonging to $A_j$ such that $i<j$, then the expected time to find 
the optimum is $E(T) \leq \sum_{i=1}^{m-1}1/p_i$. 

$\textsc{RLS}$ flips exactly 1 bit of the current solution 
to sample a new search point, compares it with the current solution 
and continues with the new one unless it is worse.
The artificial fitness levels method for the traditional static HMP operator from~\cite{CorusOlivetoYazdani2019TCS} 
states that any upper bound on the expected runtime of $\textsc{RLS}$   proven using the 
artificial fitness levels (AFL) method 
also holds for the \oneoneiahype{} multiplied by 
an additional factor of $n$ (i.e., the algorithm is at most a linear factor slower than $\textsc{RLS}$  for problems where the original upper bound is tight). 
The result was shown to be tight for some standard benchmark functions including \textsc{OneMax} and \textsc{LeadingOnes}.
We will now extend the methodology to also hold for the {\it fast} HMP operator defined in the previous section by
establishing a relationship between the upper bounds on the expected runtimes of $\textsc{RLS}$ achieved via AFL  and those of the \fastia{}. 
However,  these upper bounds will differ only by a factor of 
$O(1+\gamma\log{n})$ instead of $n$. Thus, for values of $\gamma=O(1/\log n)$, the upper bounds of the two algorithms are asymptotically the same, and the methodology will allow to prove a linear speed up for the {\it fast} HMP operator compared to traditional static HMP for the cases where the AFL methodology from~\cite{CorusOlivetoYazdani2019TCS} is tight.

We start our analysis by relating the expected number of fitness function 
evaluations to the expected number of \fhm{} operations until an optimum 
is found. 
The lemma quantifies the number of expected fitness function evaluations performed by the two operators in one hypermutation.

\begin{restatable}{lemma}{wald}\label{lem:wald}
 Let $T$ be the random variable denoting the number of applications of \hypfcm{} with parameter $0<\gamma<1$ 
 until the optimum is found. 
 Then, the expected number of function 
evaluations in an \hypfcm{} 
operation given that no improvement is found is in the order of 
$\Theta(1+\gamma\log{n})$. Moreover, the  expected number of total function evaluations is at 
most $O(1+\gamma\log{n}) \cdot E[T] $.
\end{restatable}

\begin{proof}
Let the random variable $X_i$  for $i \in [T]$ denote the number of fitness 
function evaluations during the $i$th execution of a \fhm. Additionally, let 
the random variable $X_{i}'$ denote the number of fitness function evaluations 
at the $i$th operation assuming that no improvements are found.  For all $i$ it holds 
that  $X_i \leq X_{i}'$ since finding an improvement can only decrease the 
number of evaluations. 
Thus, the total number of function evaluations $E[\sum_{i=1}^{T}X_i]$
can be bounded above by $E[\sum_{i=1}^{T}X_{i}']$ which is equal to $E[T]\cdot 
E[X']$ due to Wald's equation \cite{MitzenmacheUpfal} 
 since all $X_{i}'$ are identically distributed and 
independent from $T$.

We now write the expected number of fitness function evaluations in each operation as the 
sum of $n$ indicator variables $Y_i \in \{0,1\}$ for $i\in[n]$ denoting whether an 
evaluation occurs right after the $i$th bit mutation. Referring to the probabilities 
in~(\ref{probGamma}), we get
$
 E[X']=E\left[\sum\limits_{i=1}^{n}Y_i\right]=\sum\limits_{i=1}^{n} Pr\{Y_i=1\}
 = \frac{1}{e} + \frac{1}{e}+2\sum\limits_{i=2}^{n/2}  \frac{\gamma}{i} =
\frac{2}{e} + 2\gamma\Theta(\log{n})
=\Theta(1+\gamma \log n)$. 
The second statement is obtained by multiplying this amount with $E[T]$.
\end{proof}
%

In Lemma~\ref{lem:wald}, the evaluation parameter $\gamma$ appears as a multiplicative factor 
in the expected runtime measured in fitness function evaluations. 
An intuitive lower bound of $\Omega(1/\log{n})$ for $\gamma$ can be inferred 
since smaller $\gamma$ will not decrease the expected runtime. Nevertheless, in Section~\ref{sec:fastoptia} we will provide an example where 
a smaller choice of $\gamma$ reduces $E[T]$ directly. For the rest of our 
results though, we will rely on $E[T]$ being the same as for the traditional HMP with FCM
while the number of wasted fitness function evaluations decreases from $n$ to 
$O(1+\gamma\log{n})$.

We now present the main result of this section. 
%
%
 %
The theorem applies to $(1+1)$~frameworks using the \hypfcm{} 
as hypermutation operator. 
\begin{restatable}{theorem}{aflk} \label{th:aflk}
Let $E\left(T^{AFL}_{A}\right)$ be any upper bound on the expected 
runtime of  algorithm A established by the artificial fitness levels method. 
Then,
\begin{align*}
&E\left(T_{\text{\hypfcm{}}}\right) 
\leq E\left(T^{AFL}_{RLS}\right) \cdot O(1+\gamma\log{n}).
\end{align*}
 \end{restatable}

\begin{proof}
The upper bound on the expected runtime of RLS to solve any function obtained by applying AFL  is $ E(T_{RLS}^{AFL}) \leq \sum_{i=1}^m 1/p_i$, where $p_i$ is $s/n$ when all individuals in level $i$ have at least $s$ Hamming neighbours which belong to a higher fitness level. The probability of mutating to one of the solutions in the first mutation step is the same for \hypfcm. Such a solution will be evaluated with probability $1/e$. If a solution is not found in the first mutation step, then according to Lemma~\ref{lem:wald} at most $O(1+\gamma \log n)$ fitness function evaluation would be wasted. Since the algorithm is elitist and only accepts individuals of equal or better fitness, each level has to be left only once, 
independent of whether improvements are achieved by one or more bit-flips. Hence the claim follows.
\end{proof}


Apart from showing the efficiency of the \fastianp, the theorem also allows to easily achieve 
upper bounds on the expected runtime of the algorithm by just analysing the simple \textsc{RLS}. 
For $\gamma=O(1/\log{n})$, Theorem~\ref{th:aflk} implies the upper bounds of 
$O(n\log{n})$ and  $O(n^2)$ for classical benchmark functions \textsc{OneMax}$(x) = \sum_{i=1}^n x_i$  
and \textsc{LeadingOnes}$(x)=\sum_{i=1}^{n} \prod_{j=1}^{i} x_j$ respectively~\cite{OlivetoBookChapter}. Both of these 
bounds are asymptotically tight since each function's unary unbiased black-box 
complexity is in the same asymptotic order~\cite{Lehre2012}. 
These expected runtimes represent linear speed-ups compared to the (1+1)~IA using the static HMP operators from the literature which have 
$\Theta(n^2 \log n)$ and $\Theta(n^3)$ expected runtimes  for \textsc{OneMax} and \textsc{LeadingOnes} respectively~\cite{CorusOlivetoYazdani2019TCS}.
\begin{corollary}\label{cor:onemax}
The expected runtime of the \fastianp using  \hypfcm{} to optimise 
$\textsc{OneMax}(x):=\sum_{i=1}^{n}x_i$ and 
$\textsc{LeadingOnes}:=\sum_{i=1}^{n}\prod_{j=1}^{i}x_j$ is  
respectively $\Theta\left(n \log{n}\left(1+\gamma \log{n}\right)\right)$ and 
$\Theta(n^2 \left(1+\gamma \log{n}\right))$. For $\gamma=O(1/\log{n})$ these 
bounds reduce to $\Theta(n \log{n})$ and $\Theta(n^2)$. 
\end{corollary}

\section{Fast Hypermutations for Standard Multimodal Benchmark Functions}
In the previous section we showed that linear speed-ups compared to static HMP are achieved by the \fastianp for standard unimodal benchmark functions i.e., the algorithm is fast at exploitation for hill-climbing problems. In this section we will show that exponential speed-ups compared to the standard bit mutation operators used in traditional EAs are still achieved for standard multimodal benchmark functions i.e., the Fast HMP operators are also efficient at exploration.


We start by using the mathematical methodology derived in the previous section 
to show that the \fastianp is even faster than static HMP for the deceptive \textsc{Trap} function which is 
identical to \textsc{OneMax} except that the optimum is in $0^n$.
\hypfcm{} samples the complementary bit-string with 
probability one if it cannot find any improvements. This behaviour allows it to be 
efficient for this deceptive function. 
Since $n$ bits have to be flipped to reach the global optimum from the local optimum, EAs with  SBM
require exponential runtime with overwhelming probability (w.o.p.)\footnote{In this paper we consider events to occur  ``with overwhelming probability'' (w.o.p.) meaning that they occur with probability at least $1 - 2^{-\Omega(n)}$.}~\cite{OlivetoBookChapter}. By  
evaluating the sampled bit-strings stochastically, the \fastianp provides up to a linear speed-up for small enough $\gamma$ compared to the  \oneoneiahype{}
on \textsc{Trap} as well. 

\begin{restatable}{theorem}{trap}
\label{thm:trap}
The expected runtime of the \fastianp for \textsc{Trap} is 
$\Theta(n \log{n} \left(1+\gamma \log{n}\right))$.
\end{restatable}
\begin{proof}
According to Corollary~3 in the main document, 
 we can 
conclude that the current individual will reach $1^n$ in $ O(n \log n 
\cdot(1+\gamma \log n))$ steps in expectation. 
The global optimum is found in a single mutation operator with probability $1/e$ by 
evaluating after flipping all bits for which the number of additional fitness 
evaluations is $O(1+\gamma \log n)$ in expectation. This bound is asymptotically tight since the function's unary unbiased black-box complexity is in the same order as the presented upper bounds \cite{Lehre2012}.
\end{proof}


The  results of the \oneoneiahype{} on $\textsc{Jump}_d$  and $\textsc{Cliff}_d$
functions \cite{CorusOlivetoYazdani2019TCS} can also be adapted to the \fastianp in a straightforward manner. 
%

Both $\textsc{Jump}_d$ 
and $\textsc{Cliff}_d$  have the same structure as \textsc{OneMax} for bit-strings 
with up to $n-d$ 1-bits and the same optimum $1^n$. For solutions with the 
number of 1-bits between $n-d$ and $n$, $\textsc{Jump}_d$ has a reversed 
\textsc{OneMax} slope creating a gradient towards $n-d$ while $\textsc{Cliff}_d$ 
has a slope heading toward $1^n$, but the fitness values are penalised by an additive factor $d$.
These functions are illustrated in Fig. \ref{fig:JumpCliff}.
Since hypermutation operators have a higher probability of flipping multiple bits, the performance of static HMP
on the $\textsc{Jump}_d$ and $\textsc{Cliff}_d$ functions is superior to that of the standard bit 
mutations used by traditional EAs~\cite{CorusOlivetoYazdani2019TCS}. This advantage is preserved for the \fastianp as shown by the following 
theorem.

\begin{figure}[t!]
 \centering
  \includegraphics[width=.5\textwidth]{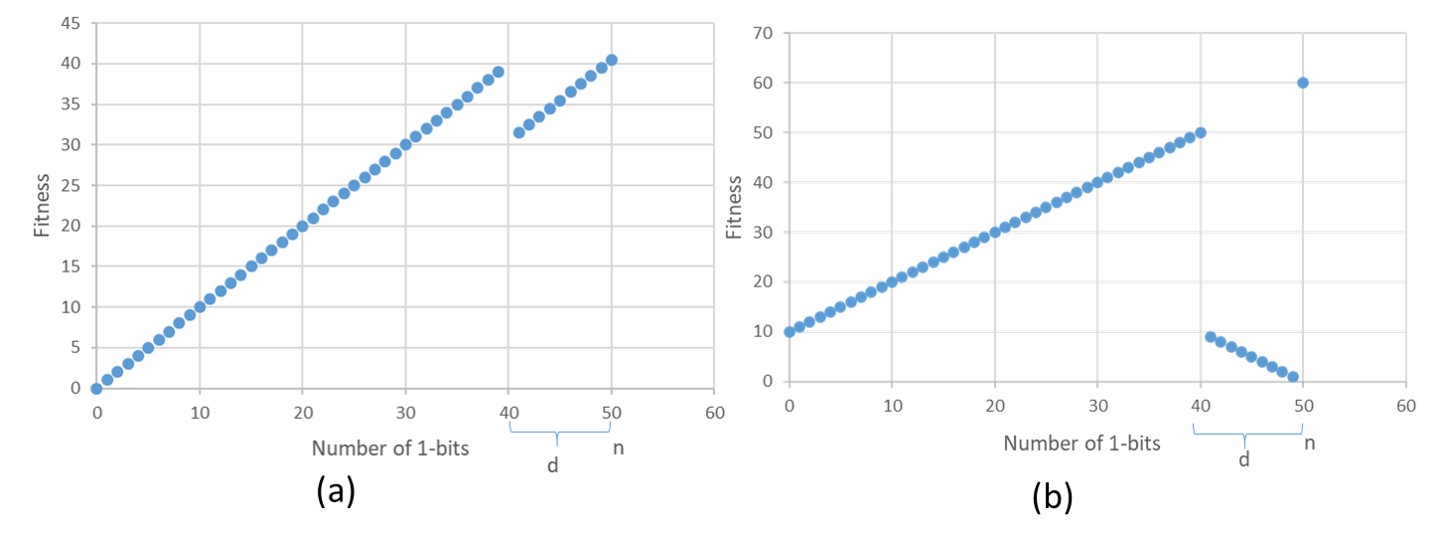}
 \caption{(a) \textsc{Cliff}$_d$ and (b) \textsc{Jump}$_d$ for $n=50$ }
  \label{fig:JumpCliff}
 \end{figure}

%
%

\begin{restatable}{theorem}{jump}
The expected runtime of the \fastianp for $\textsc{Jump}_d$ and 
$\textsc{Cliff}_d$ is $O\left(\binom{n}{d}\cdot (d/\gamma) \cdot (1+\gamma \log n)\right) $. 
\end{restatable}

\begin{proof}
According to Corollary~3, 
the time to sample a solution with $n-d$ 1-bits is at most $O\left(n\log{n}\left(1+\gamma \log{n}\right)\right)$ because the function behaves as \textsc{OneMax} for solutions with less than $n-d$ 1-bits. The Hamming distance of locally optimal points to the global 
optimum is $d$, thus, the probability of reaching the global optimum at the $d$th
mutation step is $\binom{n}{d}^{-1}$ while the probability of evaluating it is $\gamma/d$. Using Lemma~\ref{lem:wald}, we bound the total expected time to optimise \textsc{Jump} and \textsc{Cliff} by $E(T)= O\left(\binom{n}{d}\cdot (d/\gamma) \cdot (1+\gamma \log n)\right) $.
\end{proof}


For \textsc{Jump}$_{d}$ and \textsc{Cliff}$_{d}$, the superiority of the \fastianp in comparison to 
the deterministic evaluations scheme (i.e., the original \oneoneiahype{}) depends on the function parameter $d$. 
If $\gamma=\Omega(1/\log{n})$, the \fastianp performs better 
when $\min{\{d, n-d\}}=o(n/\log{n})$ while the  deterministic scheme  is preferable for larger $\min{\{d,n-d\}}$. 
However, for small $\min{\{d,n-d\}}$ the difference between the runtimes can be as large as 
a factor of $n$ in favor of the \fastianp, while even for the largest $\min{\{d,n-d\}}$, the 
difference is less than a factor of $\log{n}$ in favor of the deterministic 
scheme. Here we should also note that when both $d$ and $n-d$ are in the order of $\Omega(n/\log{n})$, the expected 
time is exponentially large for both algorithms (albeit considerably smaller than that of standard EAs) and the $\log{n}$ factor has no 
realistic effect on the applicability of the algorithm. 
For these reasons the \fastianp should be more efficient in practice.

\section{\fastoptia}\label{sec:fastoptia}
In the previous sections we showed how the \fastianp achieves linear speed-ups in the exploitation phases compared
to the traditional static HMP, while still maintaining a high quality performance at escaping from local optima of multimodal functions.
In this section we will show how also the complete \fastoptia, which uses a population, cloning, hypermutations and an ageing operator, can take considerable advantage from the use of the {\it fast} HMP operator.
In particular, we show linear, quasi-linear and exponential speed-ups compared to bounds on the expected runtime of the standard Opt-IA known in the literature.

\subsection{Optimal Expected Runtimes for Unimodal functions}
We start by analysing 
the performance of the \fastoptia for standard unimodal benchmark functions, i.e., \textsc{OneMax} and \textsc{LeadingOnes}.
Essentially the bounds derived previously for the \fastianp also apply to the \fastoptia by multiplying them with the population and clone sizes as long as the
parameter $\tau$ is set large enough such that ageing does not trigger with overwhelming probability before the global optimum is identified
(i.e., the use of ageing does not make sense unless local optima are identified first).
Hence, for correctly chosen parameter values, the algorithm can optimise these unimodal functions in optimal asymptotic expected runtimes.

\begin{theorem} \label{th:FoptOneMax}
 \fastoptia{} with parameters $\mu\geq 1$, $dup \geq 1$ and $\tau=c \cdot n \log n$ for some constant $c$, optimises  \textsc{OneMax} and \textsc{LeadingOnes} in   expected $O(\mu \cdot dup \cdot n \log n \cdot(1+\gamma \log n))$ and $O(\mu \cdot dup \cdot n^2 \cdot(1+\gamma \log n))$ fitness function evaluations respectively.
\end{theorem}

\begin{proof}
For \textsc{OneMax}, we pessimistically assume that only one individual makes progress and that it only does so in the first bit flip of the hypermutation. Let $i$ be the number of 0-bits in the considered individual. Then in at most $\sum_{i=1}^n en/i=O(n \log n)$ generations it will find the optimum. Taking into account that 
in each hypermutation the expected  fitness evaluations wastage in case of failure is $O(1+\gamma \log n)$  and that dup $\cdot \mu$  individuals are hypermutated in each generation, we get $O(\mu \cdot dup \cdot n \log n \cdot(1+\gamma \log n))$  as an upper bound on the expected runtime. 

We show the upper bound for \textsc{LeadingOnes} with the same pessimistic assumptions. Since the 
probability of improving in the first bit flip of the hypermutation is $\frac{1}{en}$, we get a bound of $O(\mu 
\cdot dup \cdot n^2 \cdot(1+\gamma \log n))$ on the expected number of fitness function evaluations since at most $n$ improvements have to be made. 

Since for both problems the improvement probability is at least $\frac{1}{en}$ in each iteration, the  probability that the waiting time for an improvement is at least $n\sqrt{\log{n}}$ is at most $e/\sqrt{\log{n}}$  by Markov's inequality. Thus, the probability that the best individual in the population reaches age $\tau=\Theta(n\log{n})$ is at most $\sqrt{\log{n}}^{-\Omega(\sqrt{\log{n}})}$. 
\end{proof}

%
%

\subsection{Quasi-linear Speed-Ups when Both Hypermutations and Ageing are Necessary: \textsc{HiddenPath}}\label{sec:hiddenpath}
 In \cite{CorusOlivetoYazdani2019TCS}, a benchmark function called \textsc{HiddenPath} (Fig. \ref{hp}) 
was presented where the use of both the ageing and the hypermutation operators is crucial for finding 
the optimum in polynomial time. 
\textsc{HiddenPath} is defined as
 \[
\textsc{HiddenPath}(x)= \\
\]
\[\begin{cases}
		n-\epsilon + \frac{\sum_{i=n-4}^n (1-x_i)}{n}  & \text{if}\; 
|x|_0=5 \; \& \; x \neq 1^{n-5}0^5,\\
        0 & \text{if}\; |x|_0<5 \;\text{or}\; |x|_0=n, \\
       n-\epsilon+\epsilon k/\log n & \text{if } 5 \leq k \leq \log{n}+1\; \& 
\; x=1^{n-k}0^k, \\
       n & \text{if}\; |x|_0=n-1,\\
    |x|_0 & \text{otherwise},
 \end{cases}
 \label{hiddenpath}
\]
where $|x|_0$ and $|x|_1$ respectively denote the number of 0-bits and 1-bits in a bit-string $x$.
 This function provides a gradient (where the fitness is evaluated by \textsc{ZeroMax}$=\sum_{i=1}^n (1-x_i)$) to local optima (i.e., solutions with $n-1$ 0-bits), from which the 
hypermutation operator can find another gradient (solutions with exactly five 0-bits with fitness increasing with more 0-bits in the rightmost five bit positions). This second gradient leads to a path which consists of $\log{n}-3$ solutions of the form $1^{n-k}0^{k}$ for $5\leq k \leq 
\log{n} +1$ and ends up on the global 
optimum. This path (called \textsc{Sp}) is situated on the opposite side of the search space (i.e., nearby the 
complementary bit-strings of the local optima) so it can easily be reached with hypermutations. However, the ageing operator is 
necessary for the algorithm to accept a worsening; otherwise \textsc{Sp} is not accessible because the second gradient and the \textsc{Sp} path have lower fitness than that of the local optima. 

In~\cite{CorusOlivetoYazdani2019TCS}, an upper bound of $O(\tau \mu n+\mu n^{7/2})$ for the expected runtime of the traditional Opt-IA for the problem was established.
The same proof strategy allows us to show an upper bound smaller by an $n/\log{n}$ factor for the \fastoptia.  The smaller bound is achieved thanks to the speed-up that the {\it fast} HMP operator has in the exploitation phases. The speed-up is only quasi-linear rather than linear because of the $\gamma/2=1/2\log n$ probability of evaluating a successful 2-bit flip on the $S_5$ gradient leading towards the hidden path (i.e., hence the extra $O(\log n)$ term in the upper bound).



\begin{figure}[t!]
\centering
  \includegraphics[width=.3\textwidth]{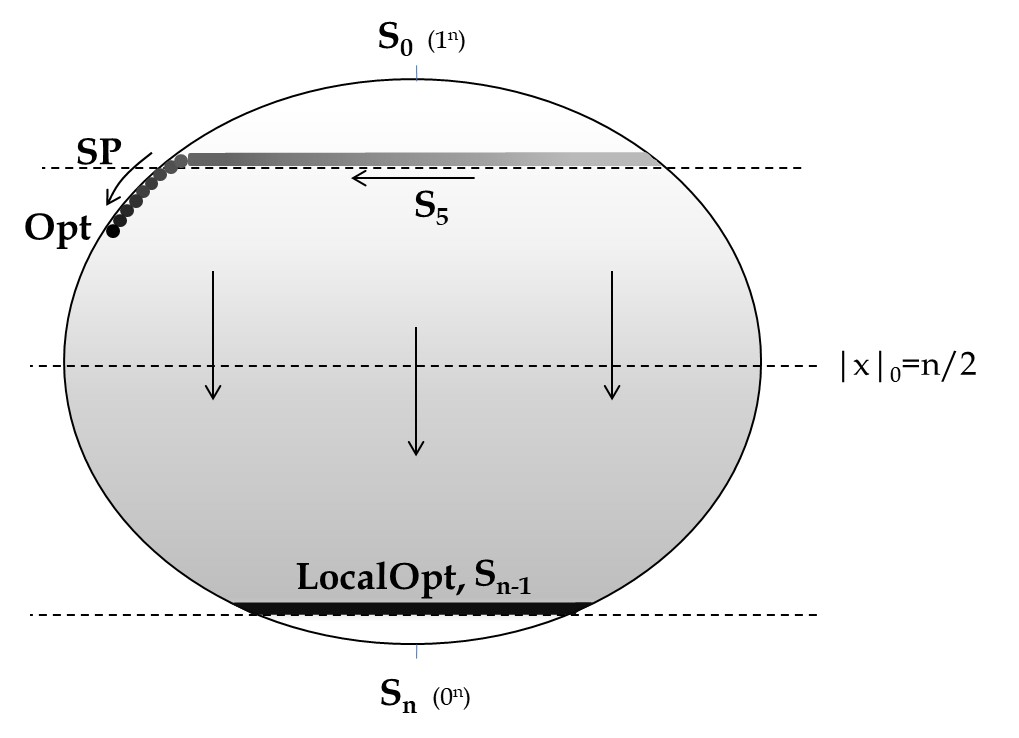}
 \caption{\textsc{HiddenPath} \cite{CorusOlivetoYazdani2019TCS}}
 \label{hp}
 \end{figure}
 
\begin{restatable}{theorem}{hiddenpath}
\fastoptia requires $O(\tau \mu + \mu n^{5/2}\log{n})$ fitness function 
evaluations  in expectation to optimise \textsc{HiddenPath} with
$\mu=O(\log n)$, {\color{black}$dup=1$}, $\gamma=\Omega(1/\log{n})\leq 1/(\color{black}5\ln{n})$ and 
$\tau=\Omega(n (\log n)^3)$. \label{th:hiddenpath}
\end{restatable}
\begin{proof}
We follow similar arguments to those of the proof of Theorem 11 in 
\cite{CorusOlivetoYazdani2019TCS} for the traditional Opt-IA. 
 For simplicity during the analysis, we call a non-\textsc{Sp} point an $S_i$ solution where $i$ is the number of 0-bits. We also pessimistically assume that until the very end, the global optimum point is not evaluated. 

After initialisation, an $S_{n-1}$ solution will be found in expected {\color{black}$O(\mu n \log n)$} 
generations by hill-climbing the \textsc{ZeroMax} part of the function according to Theorem \ref{th:FoptOneMax}. This 
individual creates and evaluates another $S_{n-1}$ search point with probability $\gamma/(2n)$ (i.e., with probability $(n-1)/n$ a 0-bit is flipped and then the 1-bit is flipped with probability $1/(n-1)$, and the solution will be 
evaluated with probability $\gamma/2$ after the second bit flip). Hence, after at most $\mu 
\cdot O(n)$ generations in expectation, the whole population will consist only of $S_{n-1}$  
solutions.
Considering that the probability of 
producing two $S_{n-1}$ solutions in one generation is $\binom{\mu}{2}\cdot O(\gamma/n) \cdot 
O(\gamma/n)=O(1/n^2)$, with probability at least $1-o(1)$ we see at most one 
new $S_{n-1}$ per generation for any phase length of $o(n^2)$ generations. Taking into 
account that the probability of creating a new $S_{n-1}$ individual is $\gamma/(2n)$ and following 
the proof of Theorem 11 in \cite{CorusOlivetoYazdani2019TCS}, we can conclude that in $O(\mu^3 \cdot 
n/\gamma)$ 
generations in expectation, the whole population reaches the same age while on 
the local optimum. Using Markov's inequality iteratively we can bound the probability that a 
population that consists of $S_{n-1}$ individuals of the same age will be observed in at most  
$O(\mu^3 \cdot n (\log{n})^2)$ generations with probability $1-o(1)$. Then, 
after at most $\tau$ generations, with probability 
$(1-1/(2\mu))^{2\mu-1} \cdot 1/(2\mu)$, one solution 
survives and the rest are removed from the population. 
{\color{black}In the following generation, while $\mu-1$ randomly initialised solutions are 
added instead of the removed solutions,
 the survived solution creates and 
evaluates an $S_1$ solution with probability $1/e$.
If an $S_1$  solution is not created before the newly generated individuals 
reach the $S_{n-1}$ level, we repeat the arguments starting from the takeover of the population by 
$S_{n-1}$ individuals. Since the jump to $S_1$ occurs with probability $1/e$, in expectation we repeat the same process at most a constant number of times and the runtime until the success has the same asymptotic order as the runtime until the first attempt, i.e., $O(\mu^3 \cdot n (\log{n})^2)$ hypermutation operations.
 
 After the $S_1$ individual is added to the population,} the hypermutation operator finds an $S_5$ solution  from this search point by flipping at most six bits and evaluating it with probability at least $\gamma/6$,  which requires $O(\log{n})$ attempts in 
expectation and in turn implies an expected time until this event occurs of $O(\mu^3 \cdot n 
(\log{n})^3)$. This individual will be added to the population with its age set to zero if the 
complementary bit-string ($S_{n-1}$) is not evaluated, which happens with probability $(1-1/e)$. In the same 
generation the $S_1$ solution dies with probability $1-\frac{1}{2\mu}$ due to ageing. 

Next, we show that the $S_5$ solutions will take over the population, and the first 
point of \textsc{Sp} will be found before any $S_{n-1}$ is found. 
An $S_5$ creates an $S_{n-5}$ individual and an $S_{n-5}$ individual creates an 
$S_{5}$ individual with constant probability $1-(1/e)$ by evaluating complementary 
bit-strings. Thus, it takes $O(1)$ generations until the number of $S_5$ and 
$S_{n-5}$ individuals in the population doubles. Since  the 
total number  of $S_5$ and $S_{n-5}$ increases exponentially in expectation, in 
$O(\log{\mu})=O(\log\log{ n})$ generations the 
population is taken over by them. After the take-over, since each $S_{n-5}$ 
solution creates an $S_{5}$ solution with constant probability, in the following 
$O(1)$ generations in expectation each $S_{n-5}$ creates an $S_{5}$ solution 
which have higher fitness value than their parents and replace them in the 
population. Overall, $S_5$ solutions take over the population in $O(\log\log{ 
n})$ generations in expectation. 

For $S_5$, $\textsc{HiddenPath}$ has a gradient towards the \textsc{Sp} which favors solutions with more 0-bits in the first (the rightmost) five bit positions. Every improvement on the gradient 
takes $O((2/\gamma) \cdot n^2)$ generations in expectation since it is enough to flip a 
 precise 1-bit and a precise 0-bit in the worst case.  Considering that there are five different fitness values on 
the gradient, in $O(5 \cdot 2 \cdot n^2/\gamma)=O(n^2/\gamma)$ generations in 
expectation the first 
point of the \textsc{Sp} will be found. Applying Markov's inequality, this 
time will not exceed $O(n^{5/2}/\gamma)$ with probability at least 
$1-(1/\sqrt{n})$. 

Now we go back to the probability of finding a locally optimal point before finding an 
\textsc{Sp} point. Due to the symmetry of the hypermutation operator, the probability of creating an $S_{n-1}$ solution from an $S_{5}$ solution is identical to the  probability of creating an $S_{n-1}$ solution from an $S_{n-5}$ solution. The probability of increasing the number of 0-bits by $k$ given that the initial number of 1-bits is $i$ and the number of 0-bits is $n-i$, is at most $(2i/n)^k$ due to the Ballot theorem since each improvement reinitialises a new ballot game with higher disadvantage (see the proof of Theorem~\ref{th:nfcmneg} for a more detailed argument). Thus, the probability that a local optimal solution is sampled is $O(n^{-4})$. The probability that such an event never happens before finding \textsc{Sp} is $1-o(1)$. After finding SP, in $O(n\log 
n)$ generations in expectations the global optimum will be found at the end of 
the SP. The probability of finding any locally optimal point from \textsc{Sp} is at most 
$O(1/n^4)$, hence this event does not happen before reaching the global optimum 
with probability $1-o(1)$. Overall, the runtime is dominated by 
$O(\tau+n^{5/2}/\gamma )$ which give us $O((\tau+n^{5/2}/\gamma ) \cdot 
\mu(1+\gamma
\log n))$ as the expected number of fitness evaluations. Since 
$\Omega(1/\log{n})=\gamma\leq 1/(\color{black}5\ln{n})$, the upper bound reduces to $O(\tau  \mu+\mu n^{5/2}\log{n})$.
\end{proof}

\subsection{Exponential Speed-Ups when Traditional Hypermutations are Detrimental: $\textsc{Cliff}_d$}\label{sec:cliff}

\textsc{HiddenPath} was originally designed  to highlight the behaviour of 
Opt-IA and illustrate its strengths. In particular, the function is an illustrative example problem where both hypermutations and ageing are necessary.
Indeed, it was shown that algorithms using either only the HMP operator (without ageing) or only the ageing operator (e.g., with standard bit mutation or local search) cannot optimise the function in polynomial time with very high probability~\cite{CorusOlivetoYazdani2019TCS}.

\textsc{HiddenPath} was especially designed to exploit the fact that HMP operators
only stop at the first constructive mutation, 
hence always return the complementary 
bit-string with probability $1$ unless some 
improvement over the parent is found before. 
On the other hand, by not returning solutions of lower quality apart from the complementary bit-string, 
the static HMP does not allow Opt-IA to take advantage of the power of ageing at escaping local optima in general, thus seriously limiting the potential explorative power of the algorithm. In this subsection we 
show that the \fastoptia, 
with approproate parameter values for $\gamma$ can escape from local optima by accepting a variety of solutions of  lower quality.

For this purpose, we consider the $\textsc{Cliff}_d$ benchmark function (defined in the previous section) which is traditionally used to evaluate the performance of randomised search heuristics at escaping local optima by accepting solutions of lower quality~\cite{JaegerskuepperStorch2007,Jorge2015, LissovoiOlivetoWarwicker2019}.
$\textsc{Cliff}_d$ was also used to show the power of the ageing operator in~\cite{CorusOlivetoYazdani2019TCS}.
 RLS 
and EAs using standard bit mutation coupled with ageing can escape the local optimum of $\textsc{Cliff}_d$ by using their small mutation rates to create solutions at the bottom of the cliff in the same iteration where the rest of the population dies. This allows  both algorithms to optimise the hardest $\textsc{Cliff}_d$ functions (when the gap between the local and global optimum is linear, i.e., $d=\Theta(n)$) respectively in expected runtimes of
$O(n\log{n})$ and $O(n^{1+\epsilon}\log{n})$ for any arbitrarily small 
positive constant  $\epsilon$. 
On the other hand, since static HMP with FCM does not return solutions of lower quality except for the complementary bit-string, 
the standard Opt-IA can only rely on hypermutations alone to escape from the local optima. 
Hence, the runtime is exponential in the distance between the top of the cliff and the global optimum w.o.p. 
The following theorem shows how 
for the hardest \textsc{Cliff}$_d$ instances, i.e., $d=\Theta(n)$, the \fastoptia has the best possible asymptotic expected runtime achievable by unary unbiased randomised search heuristics for any function with unique global optimum. 

\begin{restatable}{theorem}{nfcmpos}\label{thm:cliff}
\fastoptia{} with $\mu=O(\log{n})$, $dup=O(\log{n})$, $\gamma=1/(n \log^2 n)$ and 
$\tau= \Theta(n \log n)$  needs $O(\mu\cdot dup \cdot \tau\cdot\frac{n^2}{d^2}+n \log n)$ fitness function evaluations in 
expectation to optimise \textsc{Cliff} with $d \leq n/4-\epsilon$ for 
a small constant $\epsilon$. 
\end{restatable}

\begin{proof}
With $\gamma=1/(n \log^2 n) $, the expected number of fitness function 
evaluations  per iteration, $O(1+\gamma\log{n})$ (see Lemma~\ref{lem:wald}),
would be in the order of $\Theta(1)$. On the first \textsc{OneMax} slope, the 
algorithm improves by the first bit flip with probability at least 
$d/n=\Theta(1)$ and then evaluates this 
solution with probability $1/e=\Theta(1)$. This implies that the local 
optimum
will be found in $O(\mu \cdot dup\cdot n\log{n})$ 
 fitness evaluations in expectation after  initialisation. 

A solution at the local optimum can only improve by finding the unique 
globally optimum
solution, which requires the hypermutation to flip precisely $d$ 0-bits 
in the first $d$
mutation steps which occurs with probability $\binom{n}{d}^{-1}$. We 
pessimistically assume that
this direct jump never happens and assume that once a solution at the local 
optimum is added to 
the population, it reaches age $\tau$ at some iteration $t_0$. We consider the 
following chain of events 
that starts at $t_0$. First a solution with $(n-d+1)$ 1-bits 
would be added to the population with probability $(1/e \cdot d/n)$. Then, 
the locally optimal 
solutions will die due to ageing with probability $(1-\frac{1}{(dup+1)\cdot \mu})^{(dup+1)\cdot \mu-1}>1/e$ 
 while the post-cliff 
solution (i.e., solutions with more than $n-d$ 1-bits) will survive with probability $\frac{1}{(dup+1)\cdot \mu}$. In the next iteration, the 
post-cliff solution will improve the fitness with 
probability $d/n$ and hence resets its age to zero. If 
all of these events occur consecutively (which happens with probability $\Omega(\mu\cdot dup \cdot d^2/n^2)$), the algorithm can start climbing the second \textsc{OneMax} slope 
with local  moves (i.e., by considering only the first mutation steps) which are evaluated 
with constant probability. Then, the \textsc{Cliff} function is optimised in 
$O(n \log{n})$ function evaluations like \textsc{OneMax} unless a pre-cliff 
solution (i.e., a solution with less than $n-d$ 1-bits) replaces the current individual. The rest of our analysis will focus on 
the probability that a pre-cliff solution is sampled and evaluated given that 
the algorithm has a  post-cliff solution with age zero at iteration $t_0 +1$. 

If the current solution is a post-cliff solution, then the final bit-string 
sampled by the hypermutation operator has a worse fitness level than the current 
individual. {\color{black} The probability that \hypfcm{} evaluates at least one solution between 
mutation steps two and $n-1$ (event $\mathcal{E}_{nv}$), is bounded from above by 
$\sum\limits_{n-1}^{i=2} \gamma/i< 2\gamma\cdot \log{n}=2/(n\log{n})$.} We 
consider the $O(n\log{n})$ generations until a post-cliff solution with age zero 
reaches the global optimum. The probability that event $\mathcal{E}_{nv}$ never 
occurs in any iteration until the optimum is found is at least 
$(1-2/(n\log{n}))^{O(n\log{n})}= e^{-O(1)}=\Omega(1)$, a constant probability. 
Thus, every time we create a post-cliff solution with age zero, there is at 
least a constant probability that the global optimum is reached before any 
solution that is not sampled at the first or the last mutation step gets 
evaluated. The first mutation step cannot yield a pre-cliff solution, and the 
last mutation step cannot yield a solution with better fitness value. Thus, with 
a constant 
probability the post-cliff solution finds the optimum. If it fails to do so 
(i.e., a pre-cliff solution takes over as the current solution or a safe solution is not obtained
at iteration $t_0+1$), then in at most $O(n \log {n})$ 
iterations another chance to create a post-cliff solution comes up
and the process is repeated. In expectation, $O(\mu\cdot dup \cdot n^2/d^2)$ number of trials will be 
necessary until the optimum is found and since each trial takes $O(n \log {n})$ 
fitness function evaluations, our claim follows.
\end{proof}

Note that the above result requires the parameter $\gamma$ to be in the order of 
$\Theta(1/(n\log^2{n}))$, while Lemma~\ref{lem:wald} implies that 
any $\gamma=\omega(1/\log{n})$ does not decrease the expected number of 
fitness function evaluations per hypermutation below the asymptotic order of $\Theta(1)$ (i.e., the algorithm does not waste more than a constant number of evaluations in each hypermutation). 
Nevertheless, the smaller $\gamma=1/(n\log^2{n})$ is necessary for the algorithm to escape from the local optima efficiently.
In particular, it allows the algorithm to only evaluate the first and/or the last bit flip until the optimum is found with high enough probability. 
This 
in turn allows the Fast Opt-IA$\gamma$ to climb up the second slope 
before jumping back to the local optima via larger mutations. 
The following theorem rigorously proves that a very small choice for $\gamma$ in this case is necessary (i.e., 
$\gamma=\Omega(1/\log{n})$ leads to exponential {\color{black} expected} runtime). 

\begin{restatable}{theorem}{nfcmneg}\label{th:nfcmneg}
At least $2^{\Omega(n)}$ fitness function evaluations in {\color{black} expectation} 
are executed 
before  the \fastoptia{} with $\gamma=\Omega(1/\log{n})$ finds the 
optimum of $\textsc{Cliff}_d$ for $d=(1-c)n/4$, where $c$ is a constant $0<c<1$, independent of the values of $\mu$, $dup$ and $\tau$.
\end{restatable}

\begin{proof}
Consider a current solution with more than $n-d$ (i.e., 
post-cliff) and fewer than $n-d+2\sqrt{n}$ 1-bits. {\color{black}We show that w.o.p., \hypfcm{} will yield a solution with less than 
$n-d$ (i.e., pre-cliff) and more than $n-2d+2\sqrt{n}$ 1-bits before the 
initial individual is mutated into a solution with more than $n-d+2\sqrt{n}$ 
1-bits.} This observation will imply that a pre-cliff solution with better 
fitness will replace the post-cliff solution before the post-cliff solution is 
mutated into a globally optimal solution. We will then show that it is also 
exponentially unlikely that any pre-cliff solution mutates into a 
solution with more than $n-d+\sqrt{n}$ 1-bits to complete our proof.

{\color{black}We first provide a lower bound on the probability that \hypfcm{} with 
post-cliff input solution $x$ yields a pre-cliff solution with higher fitness 
value than $x$.} We start by determining the earliest mutation step $r_{min}$, where a 
pre-cliff solution with worse fitness than $x$ can be sampled. For any 
post-cliff solution 
$x$, $\textsc{Cliff}_d(x)=\textsc{OneMax}(x)-d+(1/2)$, and any pre-cliff 
solution $y$ with $\textsc{OneMax}(x)-d+1$ 1-bits, it holds that 
$\textsc{Cliff}_d(y)>\textsc{Cliff}_d(x)$. {\color{black}We obtain the rough bound of $r_{min}\geq d- 2\sqrt{n}$ 
by considering the worst-case event  that \hypfcm{} picks $d$ 1-bits to flip consecutively}. Let $\ell(x)$ denote the number of extra 1-bits a post-cliff solution has in 
comparison to a locally optimal solution (i.e, 
$\textsc{OneMax}(x)=n-d+\ell(x)$).
{\color{black}Now, we use Serfling's bound (Theorem~\ref{thm:serfling}) to show that with a constant probability 
\hypfcm{} will find a pre-cliff solution before $3 \ell(x)$ mutation steps and it 
will keep sampling pre-cliff solutions until $r_{min}$. 

 For the input bit-string of \hypfcm}, $x$, let the multiset of weights 
$C:=\{c_i | i\in [n]\}$ be defined as $c_i:=(-1)^{x_i}$ (i.e., $c_i=-1$ when 
$x_i=1$, and  $c_i=1$ when $x_i=0$). Thus, for a permutation $\pi$ of bit-flips over 
$[n]$, the number of 1-bits after the $k$th mutation step is 
$\textsc{OneMax}(x) + \sum_{j=1}^{k} c_{\pi_{j}}$  since flipping the 
position $i$ implies that the number of 1-bits changes by $c_i$. 

Let $\bar{\mu}:= (1/n)\sum_{j=1}^{n}c_i$ be the population mean of 
$C$ and $\bar{X}:=(1/3\ell 
(x))\sum_{j=1}^{3 \ell (x)} c_{\pi_{j}}$ the sample mean. Since the 
$\textsc{Cliff}$ parameter $d$ is less than $n/4$,   $$\bar{\mu}\leq (1/n) 
\left(\left(-3n/4\right)+ 
\left(n/4\right)\right)=-1/2.$$
In order to have a solution with at least $n-d+1$ 1-bits at mutation step $3 
\ell (x)$, the following must hold:
\begin{align*}
   &3\ell (x) \bar{X} \geq -\ell(x) \iff  \bar{X} \geq 
-\frac{1}{3}\\  &\implies  \bar{X} - \bar{\mu} \geq 
-\frac{1}{3}+\frac{1}{2} =\frac{1}{6} \iff \sqrt{3 \ell(x)}\left(\bar{X} - 
\bar{\mu}\right)  \\& \geq \frac{\sqrt{3 \ell(x)}}{6}.
\end{align*}

%

The probability that a pre-cliff solution will not be found in mutation step $3\ell(x)$ 
 follows from Theorem \ref{thm:serfling}, with sample mean 
$\bar{X}$, population mean $\bar{\mu}$, sample size $3\ell(x)$, population 
size $n$, $c_{min}=-1$ and $c_{max}=1$.

\begin{align*}
&Pr\left\{\sqrt{3 \ell(x)}\left(\bar{X} - 
\bar{\mu}\right) \geq  \frac{\sqrt{3 \ell(x)}}{6}\right\} \\& \leq \text{exp}\left(- \frac{2 \left(\frac{\sqrt{3 \ell(x)}}{6}\right)^2}{\left(1-\left( 
\frac{3 \ell(x) -1}{n}    
\right)\right)(1-(-1))^2}\right)
\leq e^{-\Omega\left(\ell(x)\right)}.
\end{align*}

Thus, with probability $(1-e^{-\Omega\left(\ell(x)\right)})$, we will sample 
the first pre-cliff solution after $3\ell(x)$ mutation steps. We focus our 
attention on post-cliff solutions with $1\leq \ell(x) \leq 2\sqrt{n}$ and can 
conclude that for such solutions the above probability is in the order of 
$\Omega(1)$.
Since the number of 0-bits changes by one in every mutation step, the event of 
finding a solution with at most $n-d$ bits implies that at some point a solution 
with exactly $n-d$ 1-bits has been sampled. Let $k_0\leq 3 \ell(x)$ be the 
mutation step where  a locally optimum solution is found for the first time. Due 
to the Ballot theorem the probability that a solution with more than $n-d$ 1-bits is sampled
after $k_0$ is at most $2d/n \leq 1/2$. {\color{black} So, with probability at least $1/2$, 
\hypfcm{} will keep sampling pre-cliff solutions until $r_{min}\leq 
d-2\sqrt{n}=\Omega(n)$.}
We will now consider the probability that at least one of the solutions sampled 
between $k_0$ and $r_{min}$ is evaluated. Since the evaluation decisions are 
taken independently from each other, the probability that none 
of the solutions are evaluated is 
\begin{align*}
&\prod\limits_{i=k_{0}}^{r_{min}}\left( 1 - \frac{\gamma }{
i}\right) \leq  \prod\limits_{i=3\ell(x)}^{r_{min}}\left( 1 - \frac{\gamma }{i}\right) \leq \prod\limits_{i=6 \sqrt{n}}^{r_{min}}\left( 1 - \frac{\gamma}{i}\right)   \\& \leq \prod\limits_{i=6 \sqrt{n}}^{r_{min}}\left( 1 - 
\frac{1}{(c_1 \log{n})i }\right),
\end{align*}
for some constant $c_1$ since $\gamma = \Omega(1/\log{n})$. We will 
separate this product into $\lfloor \log{(r_{min}/6\sqrt{n})} \rfloor$ smaller 
products and show that each smaller product can be bounded from above by 
$e^{-1/(2c_1\log)}$. The first subset contains the factors with 
indices $i \in \{(r_{min}/2)+1,\ldots, r_{min} \}$, the second set  $i \in 
\{(r_{min}/4)+1,\ldots, r_{min}/2 \}$ and $j$th set (for any $j \in [\lfloor 
\log{(r_{min}/6\sqrt{n})} \rfloor]$)  $ i \in \{ r_{min} 2^{-j}+1,\ldots, 
r_{min} 2^{-j+1}\}$. If some indices are not covered by these sets due to the 
floor operator,  we will ignore them since they can only make the final product 
smaller. Note that we assume any logarithm's base is two 
unless it is specified otherwise. 
\begin{align*}
&\prod\limits_{j=1}^{\lfloor \log{(r_{min}/6\sqrt{n})} 
\rfloor}\prod\limits_{i=r_{min} 2^{-j}+1 }^{r_{min}2^{-j+1}}\left( 1 - 
\frac{1}{(c_1 \log{n})i }\right)  \\&\leq \prod\limits_{j=1}^{\lfloor 
\log{(r_{min}/6\sqrt{n})} \rfloor}\left( 1 - \frac{2^{j-1}}{(c_1 
\log{n})r_{min} }\right)^{r_{min} 2^{-j}}  \\&\leq 
\prod\limits_{j=1}^{\lfloor \log{(r_{min}/6\sqrt{n})} 
\rfloor}e^{-1/(2c_1\log{n})}\\ &\leq e^{-\lfloor \log{(r_{min}/6\sqrt{n})} 
\rfloor/(2c_1\log{n})} = e^{-\Omega(1)},
\end{align*}
where in the second line we made use of the inequality $(1-c n^{-1})^{n}\leq 
e^{-c}$ and in the final line our previous observation that $r_{min}=\Omega(n)$.
This implies that at least one of the sampled pre-cliff individuals will be 
evaluated at least with constant probability. At this point we have established 
that a pre-cliff solution will be added to the population with constant 
probability if the initial post-cliff solution $x$ has a distance between 
$\sqrt{n}$ and $2\sqrt{n}$ to the local optima. 
 
{\color{black}Due to the operator stopping at the first constructive mutation the number of $1$s in a pre-cliff solution cannot be improved by more than one in a single hypermutation operation. 
This implies that it is 
exponentially unlikely that a pre-cliff solution is mutated into a solution with 
more than $n-d+\sqrt{n}$ 1-bits. Moreover, the fitness value of any post-cliff solution cannot be improved by  more than one either. Thus, it takes at least $\Omega(\sqrt{n})$ hypermutations until a post-cliff solution has more than $n-d+2\sqrt{n}$ $1$-bits. Since we established that a pre-cliff solution is evaluated with constant
probability at each iteration, we can conclude that at least one such individual is sampled
in $\Omega(\sqrt{n})$ iterations with overwhelmingly high probability. }
Since the \fastoptia{} cannot follow the post-cliff gradient to the optimum with overwhelmingly high
probability, it relies on making the jump from local optima to global optima. Given an initial solution with $y\in \{(n/3),\ldots, n-d+2\sqrt{n}\}$ 1-bits, the probability of jumping
to the unique global optimum is $2^{\Omega(-n)}$ as well, thus our claim follows.
\end{proof}

While the low parameter value allows the algorithm to escape from local optima as proven in Theorem~\ref{thm:cliff}, with such $\gamma$-values the hypermutation is in essence switched off, i.e., with high probability the algorithm only evaluates the first bit flip and the last one, with the latter being unlikely to be useful very often. We will address the problem again in Section~\ref{sec:optimalhyp}, when discussing the best possible fitness evaluation distribution for the {\it fast} HMP operator for general purpose optimisation.

\section{Fast Hypermutations for Combinatorial Optimisation}
In the previous sections we used standard benchmark functions from the literature to show the speed-ups 
that can be achieved in the exploitation phases with the {\it fast}  HMP operator while still maintaining excellent exploration capabilities at escaping local optima.
In this section we will validate the gained insights using classical problems from combinatorial optimisation for which the performance of the traditional EAs and AISs is known in the literature. 

In the following section we analyse the performance of the \fastianp for the NP-Hard \partition problem. 
Static HMP operators allow AISs to efficiently find arbitrarily good constant approximations for the problem~\cite{CorusOlivetoYazdaniAIJ2019}. This is achieved by escaping local optima of low quality by flipping approximately half of the bits.
Given that the parabolic distribution of the {\it fast} HMP operator decreases the probability of evaluating solutions as the $n/2$th bit flip is approached, it wouldn't be surprising if the \fastianp was to struggle on this problem. Nevertheless, we will present the remarkable result that a  linear factor smaller upper bound on the expected runtime can be achieved by the algorithm compared to the static HMP even in this apparently unfavourable scenario. 
This result shows that the insights gained from the analysis of \textsc{HiddenPath}, that overall speed-ups may be achieved for multimodal problems by emphasising the exploitation strength may also appear in classical NP-Hard problems with numerous real-world applications.

In Section~\ref{sec:vc}, we turn to the \textsc{Vertex Cover} problem. We will rigorously prove linear speed-ups of the \fastianp to identify feasible solutions to the problem compared to static HMP using a node-based representation, and for identifying 2-approximations for any instance of the problem using an edge-based representation. Thus, the analysis confirms the greater exploitative capabilities of the Fast HMP operators.

At the end of each subsection we will also argue how the results also hold for the population based \fastoptia using ageing.


 \subsection{\partition}\label{sec:part}

\partition, or \textsc{Number Partitioning}, is a simple makespan scheduling problem where the goal is to schedule $n$ different jobs with processing times $p_1 \geq p_2 \geq 
\cdots \geq p_n$ on two identical machines in a way that the load of the fullest machine is minimized. 
It is considered to be one of the six basic NP-complete problems 
\cite{GareyJohnson1979} which arises in many real world applications such as allocation tasks in industry and in information 
processing systems~\cite{Pinedo2016,Hayes2002}.
It is known that the \oneoneea~ and RLS get stuck on approximately 4/3 approximation ratios on worst-case instances of the problem.
However, they can find a
$(1+\epsilon)$ approximation for any $\epsilon=\Theta(1)$ if an appropriate restart strategy that depends on the chosen $\epsilon$ is put in place~\cite{Witt2005}. 
On the other hand the (1+1)~IA, by using static HMP can escape the local optima where EAs and RLS get stuck, thus solving the worst-case instance in expected $O(n^2)$ time. As a result it finds arbitrarily good approximations with an expected runtime that is only exponential in $\epsilon$, i.e., it can efficiently identify arbitrarily small constant $(1+\epsilon)$ approximations in every run in expected time $O(n^3)$~\cite{CorusOlivetoYazdaniAIJ2019}.  
In the following two subsections we use the same proof techniques used in~\cite{CorusOlivetoYazdaniAIJ2019} to prove that  the \fastianp optimises the worst-case instance for EAs in expected time $O(n \log n)$ and identifies a  $(1+\epsilon)$ approximation in expected time $O(n^2)$, thus providing upper bounds that are respectively a quasi-linear and linear factors smaller than those derived for the traditional static HMP operator.

{\color{black}
First we adapt a result from~\cite{CorusOlivetoYazdaniAIJ2019} regarding the expected number of generations required by the  \fastianp to identify a local optimum from a non-locally optimal solution.
In the rest of this section we use the term {\it local optimum} to refer to solutions with a makespan that cannot be improved 
by moving one single job from one machine to the other. Moreover, let $\ell_1,\ell_2,\cdots\ell_L$ denote the local optima of a 
\partition instance where $L$ is the total number of local optima and for any $i \in L, f(\ell_i) \geq 
f(\ell_{i+1})$.
\begin{lemma}[Adapted from Lemma 2 in \cite{CorusOlivetoYazdaniAIJ2019}]\label{lem:local}
 Let $x\in\{0,1\}^n$ be a non-locally optimal solution to a partition instance 
such that $f(\ell_i)>f(x)\geq f(\ell_{i+1})$ for some $i\in[L-1]$. 
Then, the \fastianp with current solution $x$ samples a solution $y$ such that $f(y)\leq f(\ell_{i+1})$ 
in  at most $2en^2$  expected generations.
\end{lemma}
%
The proof 
is essentially identical 
to that of \cite{CorusOlivetoYazdaniAIJ2019}. The main difference is that the 
\fastianp~evaluates a solution after the first bit with probability $1/e$, rather than with probability $1$ as in static HMP. Thus, the 
resulting additional multiplicative factor of $e$ in the expected runtime.
}

\subsubsection{EA's Worst-Case Instance Class}
\begin{figure}[t!]
\centering
  \includegraphics[width=.4\textwidth]{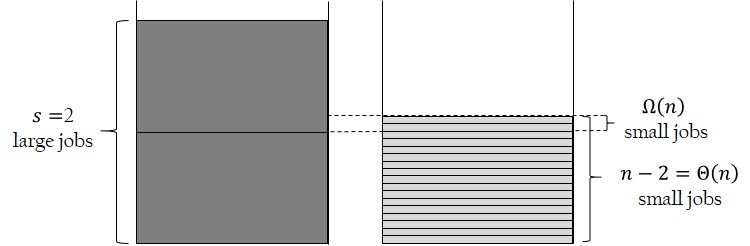}
 \caption{Worst-case approximation \partition instance, $W_{\epsilon}$, for EAs~\cite{Witt2005}.}
 \label{fig:witt}
 \end{figure}
The worst-case instance $W_{\epsilon}$  for the \oneoneea~  is depicted in Figure~\ref{fig:witt}. It consists of two large jobs $p_1$ 
and $p_2$ each with processing times $(1/3-\epsilon/4)$, and $n-2$ small jobs, $p_3, p_4, \dots, 
p_n$, each with a processing time of $(1/3+\epsilon/2)/(n-2)$. The total processing 
time is normalised between 0 and 1, and the global optima, consisting of one large job and half of the small jobs on each machine, have a makespan of $1/2$.  It has been shown that with constant probability the \oneoneea~ and RLS take 
$n^{\Omega(n)}$ fitness function evaluations to find a solution better than $(4/3-\epsilon)$ 
approximation for $W_\epsilon$ \cite{Witt2005}.

The \fastia using static HMP has been proved to be able to efficiently optimise the instance in  $O(n^2)$ expected runtime\cite{CorusOlivetoYazdaniAIJ2019}.  
The following theorem shows how the \fastianp can optimise it in $O(n \log n)$ expected function evaluations if it uses any parameter value $\gamma = \Omega(1/\log n)$. The speed-up is simply due to the fewer function evaluations wasted in the exploitation phases (i.e., it hillclimbs up to the local optima in $O(n \log n)$ expected evaluations rather than $O(n^2)$. While it is a logarithmic factor slower at escaping from the local optima, this burden does not increase the overall asymptotic order.

\begin{theorem}
The \fastianp{} optimises $W_{\epsilon}$ 
in $O(\frac{n}{\gamma} +n\log n)$ expected fitness function evaluations.
\end{theorem}

\begin{proof}
We follow the proof of Theorem 4 in~\cite{CorusOlivetoYazdaniAIJ2019}. 
For any non-optimal solution, either there is a linear number of small jobs on the fuller machine that can 
be moved to the emptier one, or the number of small jobs on each machine differ from the optimal configuration (i.e., half and half) by a linear factor (Property 1.3 in ~\cite{CorusOlivetoYazdaniAIJ2019}).
%
%
%
For the first case, the probability that the FCM$_\gamma$ operator flips a small job as the first bit-flip and evaluates the improvement is  
is $\frac{\Omega(n)}{n} \cdot \frac{1}{e}$. If this happens, then the hypermutation stops.
Since 
there are at most $O(n)$ different makespan values for the  instance class (Property 1.1 in~\cite{CorusOlivetoYazdaniAIJ2019}), the total expected number of improving hypermutation operations in these cases
is at most $O(n)$.  
For the second case, the proof in~\cite{CorusOlivetoYazdaniAIJ2019}  uses that at some point during the hypermutation the small jobs are split evenly between the two machines. When this happens, there is at least a constant probability $\Omega(1)$ that the two large jobs are on separate machines.
The probability that the \fastianp will evaluate such an optimal solution when sampled is
at least $2\gamma/n$. Hence, an optimal  solution is sampled at most $O(n/\gamma)$ times in 
expectation before it is evaluated. 

 Overall, the total number of generations before a global optimum is identified is $O(n + n/\gamma)$.
Taking into account the fitness function evaluations wasted in each generation,
i.e., at most $O(1+\gamma\log n)$ according to Lemma~\ref{lem:wald} , we get 
an upper bound of $O(1+\gamma\log n) (n + n/\gamma)  =O(\frac{n}{\gamma}+n \log n)$.
\end{proof}

\subsubsection{Worst-Case Approximation Ratio}
We now prove the main result of this subsection.
\begin{theorem}\label{thm:approxais}
The \fastianp finds a $(1+\epsilon)$ approximation for any instance of \partition in $\left[2en^2\cdot(2^{2/\epsilon}+1) + (\epsilon^{(2/\epsilon)+1})^{-1}(1-\epsilon)^{-2}e^32^{2/\epsilon}\frac{n}{2\gamma}\right]\cdot(1+\gamma \log n)$ expected fitness function evaluations for any $\epsilon=\omega(1/\sqrt{n})$. 
\end{theorem} 

\begin{proof}
The proof follows that of Theorem 6 in \cite{CorusOlivetoYazdaniAIJ2019}. We denote the 
current solution with $X_t$ and assume the algorithm stops as soon as it finds a $(1+\epsilon)$ 
approximation. 

By Lemma \ref{lem:local} we know that any non-locally optimal search point finds a  
makespan that is at least as good as the next local optimum in at most $2en^2$ generations in 
expectation by just improving in the first bit flip and evaluating the solution with probability $1/e$. As the number of local optima  with differing fitness which are not $(1+\epsilon)$ approximations is at most 
$2^{2/\epsilon}$ 
(Property 2.4 in~\cite{CorusOlivetoYazdaniAIJ2019}), 
the expected number of generations where $X_t \notin \mathcal{L}$ is at most $2en^2\cdot(2^{2/\epsilon}+1)$. 

For $\sum_{i=s+1}^n p_i \geq \frac{1}{2} \sum_{i=1}^n p_i$, any local optimum is a $(1+\epsilon)$ approximation (Property 2.3 in~\cite{CorusOlivetoYazdaniAIJ2019}). 
Therefore,  
we calculate  the expected number of generations spent with $X_t \in \mathcal{L}$ before a 
$(1+\epsilon)$ approximation is identified, assuming that $\sum_{i=s+1}^n p_i < \frac{1}{2} \sum_{i=1}^n 
p_i$. 

According to the proof of Theorem 6 in~\cite{CorusOlivetoYazdaniAIJ2019}, if the above assumption on the weights of the large jobs is in place, then the probability that in one hypermutation a $(1+\epsilon)$ approximation is sampled from any search point $X_t \in \mathcal{L}$  is 
at least 
$\frac{(\epsilon-\epsilon^2)^{2/\epsilon}}{e} \cdot \epsilon \geq \frac{\epsilon^{(2/\epsilon)+1} (1- \epsilon)^2}{e^3}$ unless an improvement is found before.
If sampled, it is evaluated by the \fastianp with probability at least 
$\frac{\gamma}{n/2}$.

Recall that there are at most $2^{2/\epsilon}$ distinct makespan values 
amongst local optima that are not $(1+\epsilon)$  approximations. 
Hence, the expected number of iterations spent on local optima is at most $(\epsilon^{(2/\epsilon)+1})^{-1} (1- \epsilon)^{-2}e^3 2^{2/\epsilon} \frac{n}{2\gamma}$.

%
%

By taking into account the expected number of wasted evaluations is each generation, the total expected runtime is at most
$\left[ 2en^2\cdot(2^{2/\epsilon}+1) +(\epsilon^{(2/\epsilon)+1})^{-1} (1- \epsilon)^{-2}e^3 2^{2/\epsilon} \frac{n}{2\gamma}\right] \cdot(1+\gamma \log n)$. 
\end{proof}

For $\gamma=1/\log n$, as recommended herein for the \fastianp, 
the expected runtime is dominated by the term $2en^2 2^{2/\epsilon}$. 
Hence the upper bound is a linear factor smaller than that of the (1+1)~IA using traditional static HMP.
We remark that even though the \fastianp is a logarithmic factor slower at escaping from the local optima,  
a speed-up is still achieved because the dominating term is due to the expected time to hill-climb up to the local optima, a task at which
the FCM$_\gamma$ operator is considerably faster. Hence, this advantage dominates even in the \partition scenario where flipping approximately $n/2$ bits is essential to escape local optima via mutation and detrimental to the \fastianp. 

We remark that the complete \fastoptia{} can also solve the worst-case instance to optimality and identify the approximation ratios by either using the ageing operator to restart the search process when trapped on local optima (with optimisation time $O(n^2)$~\cite{CorusOlivetoYazdaniAIJ2019}) or by escaping them via hypermutation. Hence, the \fastoptia{} can take advantage of both hypermutations {\it and} ageing to efficiently overcome the local optima of \partition.

\subsection{\textsc{Vertex Cover}}\label{sec:vc}

In this section we will use the NP-Hard  \textsc{Vertex Cover} problem to rigorously prove that the \fastianp can take advantage of the \hypfcm operator to achieve considerable speed-ups compared to static HMP on another classic problem from combinatorial optimisation with numerous real-world applications~\cite{VertexCoverAppBookChapter} in, e.g.,  classification methods~\cite{VertexCoverApp2014}, computational biology~\cite{VertexCoverApp2008}, and electrical engineering~\cite{VertexCoverApp1993,VertexCoverApp1998}.

Given an undirected graph $G=(V,E)$, the \textsc{Vertex Cover} problem asks to find a minimal subset of vertices, $V' \subseteq  V$, such that every edge $e \in E$ is adjacent to one of the vertices in $V'$.
Any set of vertices such that all edges in the graph are adjacent to at least one vertex in the set is a feasible solution and is called a {\it cover}. The aim of the problem is to identify the cover of minimal size (i.e., the minimum vertex cover). While the problem is NP-Hard, hence no algorithm is expected to be able to efficiently identify the optimum of every instance, we will show that  the $(1+1)$~IA using the traditional HMP operator is particularly slow at identifying any cover and how the Fast HMP operators speed-up the algorithm by a linear factor when node-based representations are used.
In the next subsection, we will prove the same linear speed-up for identifying 2-approximations when edge-based representations are employed.

\subsubsection{ Node-Based Representation}
We will use the commonly applied fitness function over node-based representations~\cite{Khuri94anevolutionary,OlivetoHeYao2009,FriedrichHe2010cover}. Candidate solutions are bit-strings of length $|V|=n$, where each bit $x_i$ is associated to a node in the graph and is set to    
$1$ 
if the vertex $i$ is included in the cover set, and to $0$ otherwise.  
The fitness of  a  candidate solution is,
\begin{equation*}
f_v(x)= \sum_{i=1}^n \left(x_i+n(1-x_i)\sum_{j=1}^n (1-x_j)e_{i,j} \right),
\end{equation*}

where $e_{i,j}$ takes value 1 if there is an edge connection between vertex $i$ and vertex $j$ 
in the graph $G$.  This fitness function sums the number of vertices in the cover (the first term) and gives a large penalty to the number of uncovered edges (the second term). 

It is well-known that both the $(1+1)$~EA and RLS can find feasible covers in expected time $\Theta(n \log n)$.
The following theorem shows that the  $(1+1)$~IA  using the traditional static HMP operator is a linear factor slower.

\begin{theorem}\label{th:IAanyVC}
The expected time until the $(1+1)$~IA finds a vertex cover using the node-based representation and $f_v$ is $\Theta(n^2 \log n)$.
\end{theorem}

\begin{proof}
To prove the upper bound, we use multiplicative drift proof idea of Theorem~1  in~\cite{FriedrichHe2010cover} 
for the \oneoneea{} and RLS. 
In particular, we will perform a drift analysis on the number of uncovered edges in the current solution.

Let $k$ denote the number of vertices that are incident to at least one uncovered edge and $u$ be  the total number of uncovered edges. 
The optimisation goal is to find the expected time until the number of uncovered edges is  $u=0$.

Looking at only the first bit-flip, the probability of improvement is at least $k/n$ and each accepted offspring decreases the number of uncovered edges by $u/k$ in expectation. The reason is that, on average, each of the $k$ vertices are connected to $u/k$ uncovered edges in expectation, hence, at the end of each improving step, the expected number of uncovered edges is at most $u_{t+1}:=(u_{t}-\frac{k}{nk}u_{t})=(u_{t}-\frac{1}{n}u_{t})=u_t(1-\frac{1}{n})$.
Hence, the drift i.e., the expected decrease of the number if uncovered edges in one step, is $\delta:=1/n$. Now, we can use the multiplicative drift theorem
(i.e., Theorem~\ref{th:multidrift})
 to compute the expected time until all the edges are covered (i.e., $u=0$). Assuming $g$ to be the number of uncovered edges, $g_{max}$ is $n(n-1)/2$ for a complete graph. Hence, we get an expected runtime of $E(T) \leq \frac{1}{\delta} (1+\ln (n\cdot(n-1)/2)\leq O(n \log n)$ to cover all edges. By pessimistically assuming that in the case of a failure at improving in the first mutation step, $n$ fitness function evaluations are wasted, the overall expected runtime is $O(n^2 \log n)$.

To prove a lower bound on the expected time to find a  cover, we assume that the given graph is complete (i.e., the number of edges is $n(n-1)/2$). The size of a cover for such a graph  is $n-1$. By Chernoff bounds, we know that w.o.p. the initialised solution includes at most $2n/3$ vertices (i.e., the number of 1-bits). To compute the expected time until all $n-1$ vertices are selected, we use the Ballot theorem. Considering the number of 0-bits as $i = q$ and the number of 1-bits as $n-i = p$, the probability of an improvement is at most $1-(p -q)/(p + q) = 1 - (n - 2i)/n = 2i/n$ by the Ballot theorem (Theorem~\ref{thm:ballot}).  
Since the operator stops at the first constructive mutation, it is necessary that at least $n/3$ improving hypermutations occur.
Hence, the expected runtime for the cover to be  identified is at least 
$\sum_{i=1}^{n/3} (\frac{n}{2i}\cdot 1+ (\frac{n}{2i}-1) \cdot n) = \Omega(n^2 \log n)$ where the second term takes into account that in the generations where an improvement is not identified $n$ fitness function evaluations are wasted.
\end{proof}

We now prove that the \fastianp is a linear factor faster.
\begin{theorem}
The expected time until the \fastianp finds a vertex cover using the node-based representation and $f_v$ is  $\Theta( n \log n \cdot  (1+\gamma \log n))$.
\end{theorem}

\begin{proof}
The proof of both upper and lower bounds is similar to that of Theorem~\ref{th:IAanyVC}. We start with the upper bound first. For the \fastianp,  the probability of improvement in the first bit-flip is $k/(ne)$ which does not change the runtime asymptotically. However, in case of failure at improving the fitness, the \fastianp wastes at most $O(1+\gamma \log n)$ fitness function evaluations in expectation. This yields an expected runtime of $E(T)=O((1+\gamma \log n) \cdot n \ln n)$. 

Regarding the lower bound, the proof  is identical to that of Theorem~\ref{th:IAanyVC} except that the expected  wastage when failing to find an improvement is $\Omega(1+\gamma \log n)$, which makes the expected runtime larger than $\sum_{i=1}^{n/3} (\frac{n}{2i}\cdot 1+ (\frac{n}{2i}-1) \cdot (1+\gamma \log n))=\Omega( n\log n  \cdot(1+\gamma \log n))$.
\end{proof}
%

\subsubsection{ Edge-Based Representation}
It is well understood that using the node-based representation of the previous subsection, RLS and EAs may get stuck on arbitrarily bad approximations for the \textsc{Vertex Cover} problem~\cite{OlivetoHeYao2009,FriedrichHe2010cover}. In~\cite{JansenOlivetoZargesFOGA2013}, it was shown that 2-approximations may be guaranteed by these algorithms if an edge-based representation is employed, such that if an edge is selected, then both its endpoints are included in the cover. For the approximation to be guaranteed, it is necessary to give a large penalty to adjacent edges, i.e., the fitness decreases considerably if adjacent edges are deselected.
Given a graph $G= (V,E)$ with $|V|=n$ and $|E|=m$ and an edge-based representation where solutions are bit-strings of length $m$, the fitness function is,
\begin{align*} 
f_e (x) = f_v (x)& + (|V|+1) \cdot (m + 1)\\ 
&\cdot |\{ (e,e') \in E(x) \times E(x) | e \neq e',  e \cap e' \neq \emptyset \}|.
\end{align*}
We will now prove that while with this representation the $(1+1)$~IA with traditional static HMP requires super-quadratic expected runtime in the number of edges to find a 2- approximation in the worst-case,  the \fastianp guarantees 2-approximations in expected time $O(m \log m)$ for any instance of the problem.

\begin{theorem}
Using the edge-based representation and fitness function $f_e$, the $(1+1)$~IA has an expected runtime of $\Omega(m^2 \log m)$ to find a 2-approximation for vertex cover.
The \fastianp finds a 2-approximation within $O((m \log m) \cdot  (1+\gamma \log m))$ expected fitness function evaluations.
\end{theorem}

\begin{proof}
For the first statement we consider  a star graph with $|V|=n$ and $|E|=m=n+1$. All nodes but one are connected to the central one with exactly one edge. We follow the proof idea for the lower bound of Theorem~\ref{th:IAanyVC}. 
By Chernoff bounds, w.o.p. the algorithm is initialised with at least $m/3$ selected edges and they all have to be deselected except for one since all the edges are adjacent. Only then will the resulting cover be a 2-approximation. By the Ballot theorem, given that $i$ edges still need to be deselected, the probability that by the end of a hypermutation an improvement is found is $2i/m$. 
The statement follows by considering that in the $m/(2i) -1$ expected generations where no improvement occurs, $m$ fitness evaluations will be wasted, and by summing these evaluations up for the $1 < i < m/3$ necessary improvements due to stopping after each constructive mutation.

The proof of the upper bound follows directly the proof for RLS and the (1+1)~EA of Theorem 11 in~\cite{JansenOlivetoZargesFOGA2013}.
Two \textsc{OneMax}-like phases suffice to guarantee that a 2-approximation is found: the first one removes all the adjacent edges, thus removing the largest penalty term completely, and the second one adds any edges connecting uncovered nodes, thus removing the penalty term of $f_v$. Since by Corollary~3 the \fastianp optimises \textsc{OneMax} in expected time $O(m \log m) \cdot  (1+\gamma \log m)$, and two such phases suffice, the proof is concluded.
\end{proof}

As long as the ageing parameter $\tau$ is set to be asymptotically larger than the expected waiting time of the improvement with smallest probability, all the results proven for \textsc{Vertex Cover} can easily been shown to also hold for the \fastoptia{}  by multiplying the upper bounds  with the population and clone sizes.

\section{Optimal Probability Distributions} \label{sec:optimalhyp}
In the following subsection we compare the advantages and disadvantages of our proposed {\it fast} HMP operators to other '{\it fast} mutation' operators from the literature. In the subsequent subsection we draw on the gained insights to provide the best parameter settings for the {\it fast} HMP operators for black box scenarios where limited problem knowledge is available.

\subsection{Comparison with Fast Evolutionary Algorithms }
While high mutation rates are typical of an immune system response, they do not occur naturally in Darwinian evolutionary processes.
Indeed, low mutation rates are essential in traditional generational evolutionary and genetic algorithms to avoid exponential runtimes on any function of practical interest~\cite{LehrTEVC2018}. 
However,  increasing evidence is mounting that higher mutation rates than standard are beneficial to steady-state GAs both for exploitation (i.e., hillclimbing)~\cite{CorusOlivetoTEVCsteadyGA2017,CorusOlivetoGecco2019} and exploration (i.e., escaping from local optima)~\cite{JumpTEVC2017}. 
These high mutation rates are possibile by taking advantage of the artificially introduced elitism in steady-state EAs~\cite{Whitley1989}.
Such insights have recently been exploited in the evolutionary computation community in the design of so-called {\it fast} EAs that use
{\it heavy tailed} mutation operators to allow a larger number of bit flips more often than the standard bit mutations (SBM) traditionally used in EAs and GAs.
 By using higher mutation rates, {\it fast} EAs can provably escape from local optima more efficiently than the traditional SBM. 
 Since these analogies are very similar to the insights gained in this paper, and in previous works regarding AISs, in this section we compare the performance of the {\it fast} HMP operator to those of the {\it fast} EAs. 
 
 Two heavy tailed mutation-based EAs for discrete optimisation have been recently introduced. In the first one, which we call Fast (1+1)~EA$_\beta$, the tail of the probability distribution follows a power law~\cite{Doerretal2017} (i.e., the probability that larger number of bits flip decreases slower than in SBM). 
 In the second one, which we call Fast (1+1)~EA$_{\textsc{unif}}$, the tail is uniformly distributed~\cite{FQWGECCO18}.
To illustrate their advantages over SBM at escaping from local optima, these works have naturally used the \textsc{Jump}$_d$ function, just like traditionally for AISs and in this paper. Thus, we will start by comparing their performance versus that of the \fastianp for \textsc{Jump}$_d$.
We begin with the latter algorithm, as the analysis of the former will motivate the optimal settings for the hypermutation distribution that we will present in the next subsection for typical black box scenarios  where minimal problem knowledge is assumed.

\subsubsection{Uniform Heavy Tailed Mutations~\cite{FQWGECCO18}} 
The Fast (1+1)~EA$_{\textsc{unif}}$ uses the following distribution:
\begin{align}
\label{prob}
p_i= 
\begin{cases}
		p & \text{for}\; i=1, \\  
        (1-p)/(n-1) & \text{for}\; 1<i \leq n.\\
\end{cases}
\end{align}
where $p_i$ is the probability that $i$ bits flip and $p=\Theta(1)$ is a constant, e.g., $1/e$.

This operator has a very similar behaviour to the original static HMP operator with FCM since over  
$n$ fitness function evaluations both operators evaluate the same expected number of solutions at Hamming distance $k$ (for any $k\neq 1$) except for a factor of $(1-p)$. 

Just like the (1+1)~IA and the \fastianp, the Fast (1+1)~EA$_{\textsc{unif}}$   
can easily explore the opposite side of the search space and can even obtain polynomial expected runtimes if the 
jump size is in the order of $n-O(1)$. However, just like for the traditional HMP operator, the drawback of this approach is that it is slower than the \fastianp for jump sizes $d<n/\log{n}$ and $d>n-n/\log{n}$. 
The intuition is that the  Fast (1+1)~EA$_{\textsc{unif}}$ assigns a constant probability $p$ only to 1-bit flips while assigning a probability in the order of $\Omega(1/n)$ to all others. 
Hence, similarly to the traditional static HMP operator, a solution at the correct distance $d$ to the 
parent is only sampled once every $n$ fitness function evaluations resulting in the same asymptotic 
performance for all possible $d>2$. 
In particular, while for small and large $d$ (where efficient performance is achievable), the detriment in performance is as large as a factor of $n$, for other values of $d$, the difference of performance is in favour of the Fast (1+1)~EA$_{\textsc{unif}}$ by at most a factor of $\Theta(\log n)$. This, however, has no realistic influence on the applicability of the algorithm, since the expected runtime to perform such jumps is exponential in the problem size in any case. 
Hence, the superiority of the \fastianp at escaping local optima, while both algorithms display the same hillclimbing performance (i.e., they both flip and evaluate exactly one bit with constant probability $p=\Theta(1)$). 

%
%


\subsubsection{Power Law Heavy Tailed Mutations~\cite{Doerretal2017}}
The (1+1)~EA$_\beta$ uses a heavy tailed standard bit mutation operator (i.e., it flips each bit with probability $\chi / n$). The mutation rate $\chi$ is sampled in each step with probability, 
\[
p(\chi) = \frac{\chi^{-\beta}}{\sum^{n/2}_{i=1}i^{-\beta}}
\]
where the parameter $\beta$ is assumed to be a constant strictly greater than $1$ to ensure 
that the sum $\sum^{n/2}_{i=1}i^{-\beta}$ is in the order of $O(1)$. 

The optimal expected runtime for \textsc{Jump}$_d$ is $\frac{n^n}{d^d (n-d)^{n-d}}$, which is achieved by using the optimal mutation rate $d/n$ which can only be applied if the jump size $d$ is known in advance.
Naturally, 
in a black-box scenario this parameter of the problem is not known to the algorithm. The above mutation operator was explicitly designed 
to  have an adequate compromised performance over all possible values of $d$.

The Fast (1+1)~EA$_\beta$ has an expected runtime 
of  $\Theta\left(d^{\beta}\binom{n}{d}\right)$ on the \textsc{Jump}$_d$ 
function which differs from the best possible expected runtime by at most a factor of $\Theta(d^{\beta-0.5})$. 
The \fastianp evaluates a solution at Hamming distance $d$ with probability $\gamma/d$ in each hypermutation, and wastes the remaining 
$\Theta(\gamma\log{n})$ expected evaluations, resulting in an expected waiting time of 
$\Theta( d\log{n} \binom{n}{d})$. Thus, 
the \fastianp has an extra $\Theta(\log{n})$ 
factor in its runtime for constant jump sizes. 
In particular, since the Fast (1+1)~EA$_\beta$ uses a power law distribution, for any jump of size 
$d=\Theta(1)$, the probability that the operator picks the mutation rate $d/n$ which gives the 
highest improvement probability is in the order of $d^{-\beta}=\Theta(1)$ when $d=\Theta(1)$. 

However, the algorithm struggles with larger jump sizes compared to the \fastianp.
This is particularly critical for very large jumps i.e., $d=n-O(1)$, where the \fastianp has polynomial expected runtime 
$O( d\log{n} \binom{n}{n-d})$ while 
 the Fast (1+1)~EA$_\beta$ has exponential runtime because it flips bits with probability at most $\chi/n= 1/2$  by design (as a larger mutation rate was deemed unnecessary in the original work). 
If the cap on the maximum mutation rate is removed (as was recently considered in~\cite{FGQWPPSN18}), the resulting operator can also achieve polynomial expected runtimes for extremely large jump sizes.  However, due to the power law distribution, the probability of flipping $n-O(1)$ bits is in the 
order of $O(n^{-\beta})$ which results in a polynomially slower expected performance to that of the \fastianp. 
This is due to the symmetric sampling distribution of the \hypfcm operator  around $n/2$ that allows considerably larger probabilities of evaluating offspring  at  distance $n-O(1)$. 

Overall, the \fastianp is asymptotically faster at escaping from local optima for all super-logarithmic jump sizes and is at most a $\Theta(\log{n})$ factor slower for small constant jumps.
In the next subsection we will show how to reduce the logarithmic factor in the \fastianp to just a constant while maintaining its advantage in the settings where it has better performance.

Nevertheless, we now show that the Fast (1+1)~EA$_\beta$ can still be very efficient in practice at escaping from local optima with large basins of attraction.
In particular, just like the \fastianp, it has an $O(n^2)$ expected runtime to find arbitrarily good constant approximations for the 
 \textsc{Partition} problem considered in Section~\ref{sec:part}. 
\begin{theorem}\label{th:part-doerr}
The \foea finds a $(1+\epsilon)$ approximation for any \partition instance in 
$2(C_{n/2}^\beta)en^2\cdot(2^{2/\epsilon}+1)+(C_{n/2}^\beta) (n( 
\epsilon-\epsilon^2))^{\beta}\cdot\epsilon\cdot(\epsilon-\epsilon^2)^{-2/\epsilon}
$ expected fitness function evaluations. (for any $\epsilon=\omega(1/\sqrt{n})$). 
\end{theorem}

\begin{proof}
We can adapt Lemma~\ref{lem:local} to the \foea by modifying the expected runtime between local optima 
to reflect the probability of flipping exactly $1$-bits. The corresponding expected runtime between local 
optima for \foea would be at most $2 (C_{n/2}^\beta) e n^2$ since it flips a single 
bit with probability at least $1/(C_{n/2}^\beta e)$.

Thus, the expected time that the \foea spends where $X_t \notin \mathcal{L}$ is 
$2C_{n/2}^\beta en^2\cdot(2^{2/\epsilon}+1)$. Now, we will bound from below the probability that the
\foea finds an approximation in time $t$  given that $X_t \in \mathcal{L}$. 
Similarly to the proof of Theorem~\ref{thm:approxais}, we will refer to Property 2.2 in~\cite{CorusOlivetoYazdaniAIJ2019} 
and establish that if $X_t \in \mathcal{L}$ then the fuller machine has no small jobs assigned to 
it unless it is already a $(1+\epsilon)$ approximation. With probability $C_{n/2}^\beta (n( 
\epsilon-\epsilon^2))^{-\beta}$ the \foea will apply standard bit mutation with mutation rate 
$\epsilon-\epsilon^2$.
Using the same notation of Theorem~\ref{thm:approxais} and assuming 
$\epsilon<1/2$, with probability at least $C_{n/2}^\beta (n( 
\epsilon-\epsilon^2))^{-\beta}(\epsilon-\epsilon^2)^{2/\epsilon}$ all large jobs will 
be assigned according to their configuration in $H$. Since all small jobs are on the same machine 
in the parent solution and each bit is flipped with probability $\epsilon-\epsilon^2$, the expected 
total weight transferred from the emptier machine to the  fuller machine is at most 
$\frac{(\epsilon-\epsilon^2)\cdot W}{2}$. The rest of the analysis is identical except for the part 
where we consider the possibility that, even though a successful hypermutation is bound to happen, 
 a prior improvement prevents the $n(\epsilon-\epsilon^2)$ bit-flips to occur. 
This scenario does not take place for the \foea because all mutating bits flip simultaneously, thus with probability 
$C_{n/2}^\beta (n(\epsilon-\epsilon^2))^{-\beta}(\epsilon-\epsilon^2)^{2/\epsilon}\cdot 
\epsilon$ an approximation is obtained and we 
do not have to pessimistically repeat this argument for all $x \in \mathcal{L}$ as we have to for the 
\fastia{}. 
\end{proof}

Even though the \foea is slower at jumping over large basins of attraction, its expected runtime for \partition is dominated by the expected time spent in the hillclimbing phases.
Indeed the bounds on the expected runtimes during exploitation of the \foea and the \fastianp are asymptotically the same
(i.e., they differ in the former having an extra
$C_{n/2}^\beta =\Theta(1)$ factor and the latter an
 extra factor of $O(1+\gamma\log{n})$ which 
is $O(1)$ for $\gamma=1/\log{n}$). 
Concerning the terms related to the expected times to escape from local optima, the \fastianp has an asymptotically smaller term of 
$2^{2/\epsilon}\cdot \frac{n}{\gamma}$ 
compared to the $(n( \epsilon-\epsilon^2))^{\beta}$ term for some constant $\beta>1$ for the \foea. We should note here that the 
$2^{2/\epsilon}$ factor (i.e., exponential in $1/\epsilon$) may appear to possibly make a crucial difference for small constant approximations in practice. However,
on one hand, this is likely to be overly pessimistic since it assumes that whenever the hypermutation is 
about to find an approximation, another improvement prevents it from flipping the necessary number 
of bits. On the other hand, the exponential factor nevertheless appears for both algorithms in the dominating term related to the hillclimbing phases.

We now highlight a huge advantage of the \foea over the {\it fast} HMP operators when escaping local optima in conjunction with ageing by accepting solutions of lower quality.
In Section \ref{sec:fastoptia} we proved that the \fastoptia optimises the Cliff$_d$ function efficiently, if the parameter $\gamma$ of the \hypfcm is set to extremely small values in the order of $\Theta(\gamma=1/n\log^2 n)$ (Theorem \ref{thm:cliff}). As a result the algorithm very rarely evaluates solutions where more than one bit is flipped i.e., it essentially does not hypermutate anymore. The following theorem shows how the \foea can optimise the function efficiently while still mutating many bits very often i.e., it  hypermutates. The result comes at the expense of slightly increasing the power law parameter to a constant $\beta > 2$ and at the expense of a square root term in the upper bound of the expected runtime instead of the logarithmic term that appears in the expected runtime of the \fastoptia with small $\gamma$. Nevertheless, although not optimal for \textsc{Jump}$_d$, with such a parameter setting the algorithm is only a constant factor slower for the \textsc{Jump}$_d$ instances for which it is very efficient (i.e., $d= \Theta(1)$).
%
\begin{theorem}\label{th:fmutcliff}
The \foea with hybrid ageing parameter $\tau= \Omega(n \log n)$  and $\beta\geq 2+\epsilon$ needs 
$O(\tau\cdot n^{3/2})$ fitness function evaluations in expectation to optimise \textsc{Cliff} with 
any linear $d \leq n(1/4-c)$ for any arbitrarily small positive constants $\epsilon$ and 
$c$. 
\end{theorem}
\begin{proof}
 The process until the cliff point is sampled for the first time is identical to the previously 
analysed algorithms with ageing. We will now consider the probability that the \foea applies a 
mutation with size at least $k$ while it is $k$ ahead of the cliff point, \emph{i.e.}, it has $d-k$ 
$0$-bits. In particular we will consider the probability that  this mutation occurs before an 
improvement is found for each $k$. We will then divide the runtime into two cases according to 
whether $d-k$ is in the order of $o(n)$ or $\Omega(n)$. For the $\Omega(n)$ regime, we will 
establish the probability of decreasing the number of $0$-bits from $d-k$ to $d-2k$ to obtain a 
lower bound on the probability of leaving the $\Omega(n)$ regime by doubling the current best 
solution's distance to the cliff edge $\log{n}$ times. Finally, we will bound the conditional 
improvement probability with the assumption $d-k=o(n)$.

We first focus on the the power law distribution of the mutation size $m$,
\allowdisplaybreaks
\begin{align*}
Pr \{ m\geq k\}&=\sum_{i=k}^{n}i^{-\beta}/C^{\beta}_{n}\\
&\leq \frac{1}{C^{\beta}_{n}}\int_{k-1}^{n}x^{-\beta}\cdot dx = 
\frac{x^{1-\beta}}{1-\beta}|^{n}_{k-1}\\
&\leq\frac{1}{\beta-1} \left( (k-1)^{1-\beta}-n^{1-\beta} \right)\\
&\leq \frac{(k-1)^{1-\beta}}{\beta-1} 
\end{align*}
We now consider the conditional probability that for the current solution $x$ with $d-k=\Omega(n)$ 
$0$-bits to apply a mutation with size larger than $k$ before it improves. The improvement 
probability is at least $\frac{d-k}{n\cdot C^{\beta}_{n}}=:p_k=\Omega(1)$ and for $k>1$, the conditional probability of improving before a large mutation is:
\begin{align*}
&\frac{p_k}{p_k+\frac{(k-1)^{1-\beta}}{\beta-1} } 
=1-\frac{\frac{(k-1)^{1-\beta}}{\beta-1}}{p_k+\frac{(k-1)^{1-\beta}}{\beta-1} } \\
&\geq 1-\frac{(k-1)^{1-\beta}}{p_k\cdot (\beta-1)} \\
\end{align*}
We will now consider probability that there will be $k$ improvements starting from a current 

solution with $d-k$ $0$-bits, but first we need to show that the above conditional probability is increasing  with $k$.

\begin{align*}
\left(1-\frac{(i)^{1-\beta}}{p_{i+1}\cdot (\beta-1)} 
\right)-\left(1-\frac{(i-1)^{1-\beta}}{p_i\cdot (\beta-1)} 
\right)\\
=\frac{(i-1)^{1-\beta}}{p_i\cdot (\beta-1)}-\frac{(i)^{1-\beta}}{p_{i+1}\cdot (\beta-1)} \\
=\frac{n\cdot C_{n}^{\beta}}{\beta-1}\cdot\left(\frac{(i-1)^{1-\beta}}{d-i}-\frac{(i)^{1-\beta}}{d-i-1} \right)
\end{align*}
Since the $\frac{n\cdot C_{n}^{\beta}}{\beta-1}$ term is positive, we are only interested in the sign of the remaining part being positive.

\begin{align*}
\allowdisplaybreaks
0&\leq\left(\frac{(i-1)^{1-\beta}}{d-i}-\frac{(i)^{1-\beta}}{d-i-1} \right)\\
\frac{i^{1-\beta}}{d-i-1}&\leq\frac{(i-1)^{1-\beta}}{d-i}\\
\frac{i^{1-\beta}}{(i-1)^{1-\beta}}&\leq\frac{d-i-1}{d-i}\\
\left(\frac{i-1}{i}\right)^{\beta-1}&\leq\frac{d-i-1}{d-i}\\
\left(1-\frac{1}{i}\right)^{\beta-1}&\leq1-\frac{1}{d-i}\\
\left(1-\frac{1}{i}\right)^{(\beta-1)\cdot \frac{i}{i}}&\leq \left(1-\frac{1}{d-i}\right)^{\frac{d-i}{d-i}}\\
e^{-\frac{\beta-1}{i}} &\leq e^{-\frac{1}{d-i}}\\
\frac{\beta-1}{i} &\geq \frac{1}{d-i}\\
\beta &\geq 1+ \frac{i}{d-i}.
\end{align*}
Thus, for $k=o(n)$ and $\beta \geq 1+\epsilon$ for an arbitrarily small constant $\epsilon$, the conditional probability of improving before flipping at least $k$ bits increases with $k$. The probability that there will be $k$ improvements starting from a current 
solution with $d-k$ $0$-bits is therefore  at least $\left(1-\frac{(i-1)^{1-\beta}}{p_k\cdot (\beta-1)} 
\right)^k \geq 1/e$ for $\beta>2-\log_{k}{p_k}$ (We can exclude any constant number of first steps 
which improves successfully with probability $\Omega(1)$ while allowing $-\log_{k}{p_k}$ to be 
arbitrarily small.). Thus, for  $d-k=o(n)$ we double the current improvement with respect to the 
cliff edge before losing our current best solution  with probability $1/e$. Since we cannot double 
our improvement more than $\log{(n/4)}$ times before $d-k=o(n)$, with probability at least 
$(1/e)^{\log{(n/4)}}=\Omega(n^{-3/2})$ we obtain a solution with $d-k=o(n)$. 

For $d-k=o(n)$, the conditional probability of improving before losing progress is:
$1-\frac{(k-1)^{1-\beta}}{p_k\cdot (\beta-1)}=1-O(n^{1-\beta})$ and for $\beta>2$ the algorithm 
climbs the OneMax slope in $O(n\log{n})$ iterations without losing progress with  probability at 
least $1-o(1)$. The only subconstant success probability after the process leaves the 
Cliff edge is $\Omega(n^{-3/2})$, thus the expected time can be bounded by $O(\tau\cdot n^{3/2})$.

%
%

 \end{proof}

\subsection{Power-Law Hypermutations}
In the previous subsection we highlighted two advantages of the power law heavy tailed mutation operator of the \foea over the {\it fast} HMP operator introduced in this paper.
Firstly, the former operator jumps out of local optima with small basins of attraction faster by a logarithmic factor at the expense of being slower for larger basins of attraction.
Secondly, it allows to escape local optima together with ageing by accepting solutions of lower fitness while still keeping quite high mutation rates i.e., the \fastoptia has to reduce it to at most that of SBM. These advantages are due to the capability of the power-law distribution of balancing well the number of large and 
small mutations.
In this subsection we will identify an ``optimal'' evaluation distribution for the {\it fast} HMP operator such that it can take advantage from the balancing capabilities of the power-law distribution while keeping its own advantages when larger basins of attraction have to be overcome.

In particular, considering the power-law distribution's poor performance for \textsc{Jump}$_d$
functions with gap sizes of $d =\Omega(\log^{\frac{1}{\beta-1}} n)$, and especially for $d=n(1-o(1))$, 
we keep the symmetry of the {\it fast} HMP operator around $n/2$ bit flips, but increase and decrease the evaluation probabilities away and to $n/2$ following a power-law. 
Just like in the Opt-IA literature, we will present variants with and without FCM and call the power-law HMPs FCM$_\beta$ and HMP$_\beta$, and the resulting algorithms Fast (1+1)~IA$_\beta$ and Fast Opt~IA$_\beta$ respectively, according to whether they use populations and ageing or not (we will see that the performance of FCM$_\beta$ and that of HMP$_\beta$ are approximately equivalent so we intentionally do not state whether the Fast (1+1)~IA$_\beta$ uses one operator or the other as it does not affect the results we present i.e., either can be used).


Recall that the parameter $\beta$ of the \foea is assumed to be a constant strictly greater than $1$ to ensure 
that the sum $\sum^{n/2}_{i=1}i^{-\beta}$ is in the order of $O(1)$. Thus, 
any  particular mutation rate $\chi$ has a probability of being picked in the order of 
$\Theta(\chi^{-\beta})$. 
Notice that if 
we were to set the parameter to 
$\beta=1$, the power-law mutation operator would have a very similar behaviour to that of the \fastianp. In particular, the resulting operator 
would pick a mutation rate $\chi$ with probability $1/(\chi\ln{n})$. 
%

Similarly, 
\hypfcm with $\gamma=1$ evaluates a solution with Hamming distance $k\neq 1$ away from the parent with 
probability $1/k$ and every call of the operator evaluates roughly $\ln{n}$ solutions in 
expectation. 
Thus, when compared over $\Theta(\log{n})$ consecutive fitness function evaluations, the 
expected number of offspring $k$ bits away from their  parent are in the same asymptotic order. 
 However, the parameter $\gamma$ of the \fastianp scales the frequency of evaluations at Hamming distance $k$ by 
the same multiplicative factor for all $k$, while the parameter $\beta$ of the \foea controls the 
emphasis on the smaller mutations. In particular, for $k \in \{2,\ldots,\frac{n}{2}-1\}$ changing 
$\beta$ changes the conditional probability of flipping $k$ bits given that either $k$ or $k+1$ 
bits are flipped, while changing $\gamma$ still conserves the ratio of sampled solutions with 
distance $k$ and $k+1$. 

These considerations lead us to believe that the ideal symmetric distribution for the HMP operator is a power-law one, where we move the probability mass further towards $\omega(1)$ bit flips, compared to the \foea:
\[
p_i := \frac{(\min \{i+1,n-i+1\})^{-\beta}}{\sum_{k=0}^{n}(\min\{k+1,n-k+1\})^{-\beta}}.
\]
Here the parameter should be set such that $\beta \geq 1$. 

With $\beta=1$, the probability distribution for  $i >1$ is identical to that of FCM$_\gamma$ for the parameter value we have used throughout the paper i.e., $\gamma= 1/\log n$. Notice that for $\beta=0$ the probability that $i$ bits flip is uniformly distributed at random i.e., the operator becomes very similar to that of the (1+1)~EA$_{\textsc{unif}}$. We have discussed why this is an inconvenient distribution in the previous section.
Note that the original heavy-tailed mutation operator first picks the mutation rate with which each bit position is flipped independently. Since we directly pick the number of bit-positions to be flipped, we assign a positive probability to not flipping any bits. This allows the operator to copy the best individuals and  plays a 
critical role in the performance of population based algorithms~\cite{Witt2006,CorusOlivetoGecco2019,LehrTEVC2018,CorusOlivetoTEVCsteadyGA2017,JumpTEVC2017}.

The operator behaviours, with and without FCM, 
are similar but not identical. While the HMP$_\beta$ operator evaluates exactly one new offspring per operation, 
the number of evaluated solutions  per hypermutation of the FCM variant, FCM$_\beta$, is randomly distributed with expectation $1$ (i.e., more than one evaluation - or zero - may occur in one hypermutation: the behaviour is exactly the same as in Definition 1 but using the power law distribution). A comparison between the power-law distributions of the mutation operators of the (1+1)~EA$_\beta$, the symmetric ones of the (1+1)~IA$_\beta$, the (1+1)~IA$_{\textsc{unif}}$ and the traditional SBM are shown in Figure~\ref{hp}.  Note that for the (1+1)~EA$_\beta$ we have extended the probability distribution range from~~\cite{Doerretal2017} to $n$ and considered the variant which flips exactly $k\in[n]$ bits after the mutation size is determined (similarly to what has been considered in~\cite{FGQWPPSN18}) rather than independently flipping all bit positions with probability $k/n$. 

\begin{figure}[t!]
\centering
  \includegraphics[width=.35\textwidth]{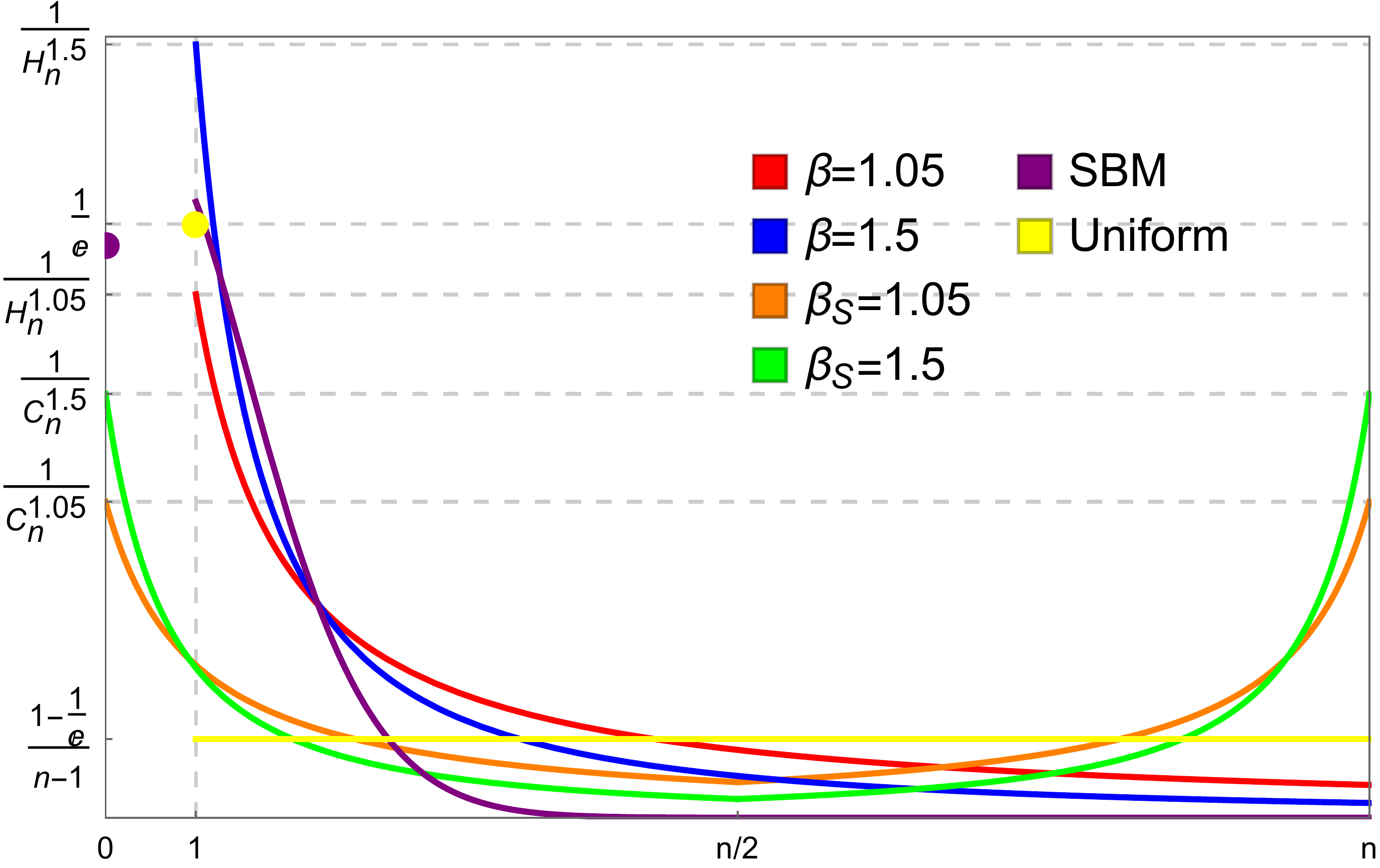}
 \caption{The probability of flipping exactly $k$ bits for the extended heavy-tailed mutation operator of Fast (1+1)~EA$_\beta$ (red and blue) and the symmetric heavy tailed mutation operator of Fast (1+1)~IA$_\beta$ (green and orange) for different $\beta$ values. The SBM used by standard EAs (purple) and the uniform heavy tailed mutation of Fast (1+1)~EA$_{\textsc{unif}}$~\cite{FQWGECCO18} with $p=1/e$ (yellow) are added for comparison. 
 The input size is set to $n=14$ for visualisation.}
 \label{hp} 
 \end{figure}

 \begin{figure}[t!]
 \centering
  \includegraphics[width=.3\textwidth]{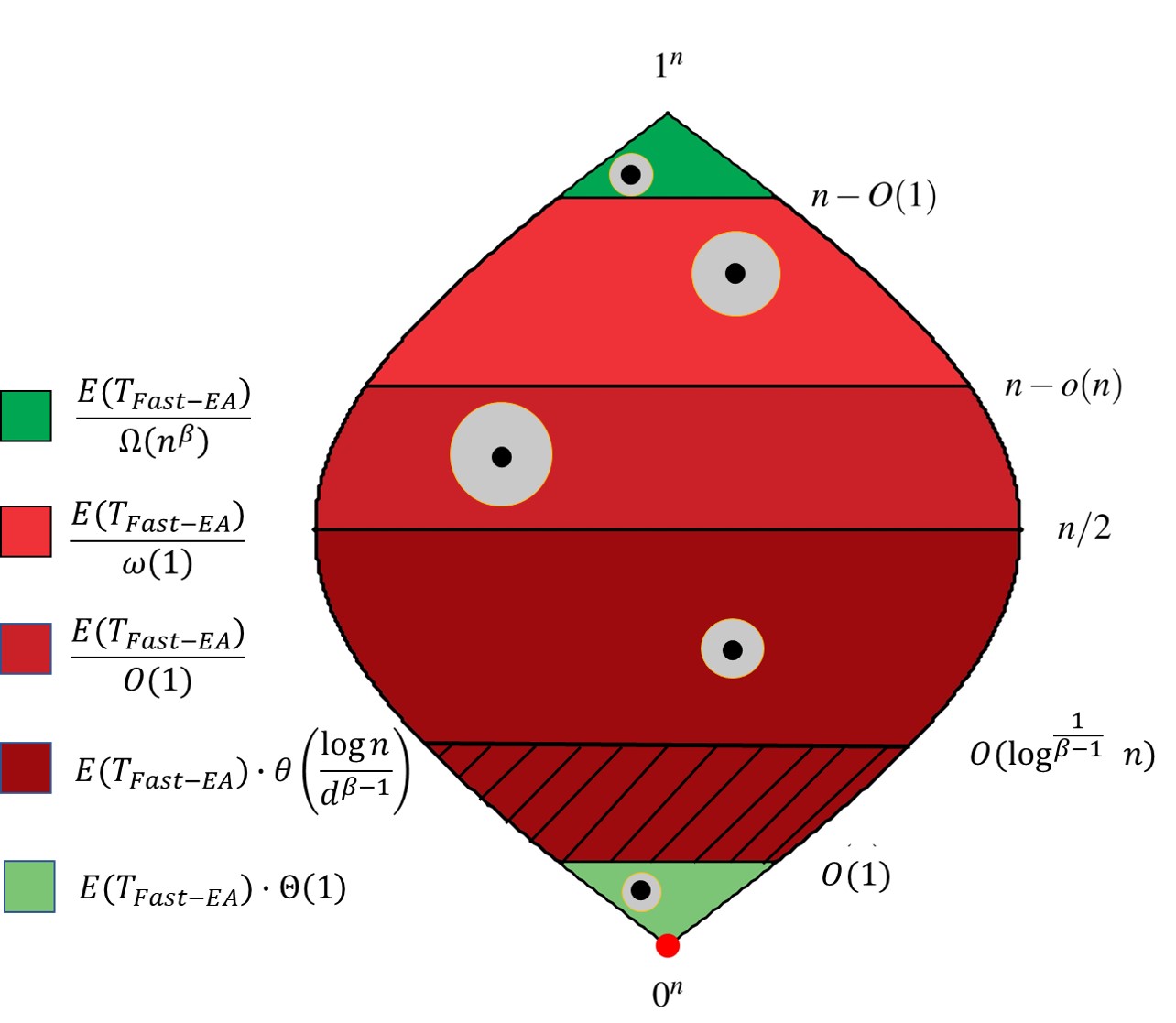}
 \caption{A description of the performance comparison of the Fast (1+1)~IA$_\beta$ and the Fast (1+1)~EA$_\beta$ at escaping from a local optimum placed on the
 hypercube at $0^n$ w.l.o.g. The global optima (and basins of attraction of any fitness quality) are located in example positions. For both algorithms the same $\beta>1$ holds for all regions except for the darkest red area. For the latter area, the Fast (1+1)~IA$_\beta$ uses the best possible parameter value $\beta=1$. For equal $\beta>1$, the Fast (1+1)~IA$_\beta$ would be a constant factor slower than the Fast (1+1)~EA$_\beta$. For both parameter setting cases, the Fast (1+1)~EA$_\beta$ asymptotically outperforms the  Fast (1+1)~IA$_\beta$ in the shaded area only.}
  \label{fig:hc}
 \end{figure}

Figure~\ref{fig:hc} shows a comparison of the expected runtimes of the (1+1)~IA$_\beta$ and the (1+1)~EA$_\beta$ to escape from local optima with different basins of attraction. Without loss of generality we assume that the local optimum is located at the $0^n$ bit-string (i.e., the red dot).
Let us denote with $y\in\{0,1\}^n$ the unique global optimum which has a higher fitness value than $x$ and $k:=HD(x,y)$. The black dots represent different potential positions in the search space for the global optimum.
The circles around the potential global optima represent basins of attraction which may or may not have higher fitness than the local optimum. These are nevertheless reachable via ageing by accepting lower quality solutions (as we have shown for \textsc{HiddenPath} and \clifff).

 Regardless of the mutation operator employed by the algorithm, the probability that $x$ is mutated into $y$ is at most $\binom{n}{k}^{-1}$ since for an unbiased mutation operator all individuals with distance $k$ to the parent have an equal probability to be sampled and $\binom{n}{k}$ is the number of individuals with Hamming distance $k$ to $x$. Note here that the binomial coefficients satisfy $\binom{n}{k}=\binom{n}{n-k}$ for all $k\leq n$. Thus, if both $k$ and $n-k$ are in the order of $\omega(1)$, the mutation probability is superpolynomially small and the jump from $x$ to $y$ has a superpolynomial expected time (i.e., the shades of red areas in the figure). Even if we relax our scenario such that the solution $y$ has a basin of attraction of a constant size, \emph{i.e.}, all individuals $z \in \{0,1\}^n$ with $HD(y,z)<d$ for some constant $d$ lead to $y$ by hillclimbing, the expected time to escape the local optima would still be super-polynomially large.  For this reason we modify the distribution over $[n]\cup\{0\}$ used to determine how many bits the heavy-tailed mutation operator will flip. We shift the probability mass from the middle to the extremities (\emph{i.e.}, from around $n/2$ to near $0$ and $n$): away from mutation sizes where a polynomial expected time is not possible.

Overall, for any $k=\Theta(1)$ the heavy-tailed mutation operator in \cite{Doerretal2017} is only faster by a constant factor than the newly suggested power-law symmetric operators at escaping the local optimum. Only for $k = O(\log^{\frac{1}{\beta -1}}n)$ (i.e., the shaded area in the figure), it is slightly asymptotically faster where both operators have super-polynomial expected runtime.  On the other hand, for all other distances of the basin of attraction of the global optimum, the symmetric power-law mutation operator is faster. In particular, the heavy-tailed operator is a polynomial factor slower than the symmetric one
when $n-k$ is in the order of $o(n)$, including for $n-k= O(1)$ where the expected runtimes of the  operators are polynomial. 
Hence, for ranges of $k$ where a polynomial expected waiting time is possible, the heavy-tailed operator of the (1+1)~EA$_\beta$ is either faster by only a constant factor than the symmetric one (i.e., when $k$ is constant) or slower by a polynomial factor (i.e., when $n-k$ is a constant). 
We point out that if in the ``super-polynomial space'' (i.e., the red areas in the figure) the basins of attraction were large enough to allow for polynomial expected waiting times, then the Fast (1+1)~IA$_\beta$ would still be faster than the \foea except for basins that fall into the diagonally shaded area.

Compared to the \fastianp, the Fast (1+1)~IA$_\beta$ is faster for all jump sizes for appropriate parameter settings 
(i.e., $\beta=1$ for $k<\log^{\frac{1}{\beta -1}}n$ and $\beta >1$ otherwise) at the expense of being a constant factor slower at hillclimbing for the suggested values of $\beta$ (i.e., close to $\beta=1$). 
In particular, the (1+1)~IA$_\beta$ is a logarithmic factor faster than the \fastianp for jumps in the ``polynomial space" (i.e., the green areas in the figure).
%

Naturally, the described above scenario also includes the behaviour on the \jumpf function. 
%
The behaviour of the \hypheavy operator for escaping local optima combined with ageing, by accepting solutions of inferior fitness, requires a more precise analysis.
Theorem~\ref{th:fmutcliff} regarding the (1+1)~EA$_\beta$ with ageing for \textsc{cliff}$_d$ relies on the distribution over the mutation rate to monotonically decrease. Since the 
distribution of the symmetric operator 
starts increasing for mutation sizes larger than $n/2$, 
the result does not transfer directly the (1+1)~IA$_\beta$. In particular, large mutation rates may lead the algorithm to jump back to the local optima once it has escaped.
Since the previous results hold for gap sizes $d \leq (1-c)n/4$ for any constant $c$, we
will show that bit flips in the order of $n(1-o(1))$ only produce solutions with smaller fitness than those observed on the second slope of the function, \emph{i.e.}, solutions with more than $n-d$ $1$-bits. Hence the operator is efficient for the function class coupled with ageing. 
The following theorem shows that \hypheavy (or HMP$_\beta$) are better suited than FCM$_\gamma$ to be used in the complete Opt-IA since they have high mutation rates (i.e., they hypermutate) and work well in harmony with ageing, 
as originally desired in the design of the Opt-IA.

\begin{theorem}\label{th:newhypecliff}
The Fast (1+1)~IA$_\beta$~ with hybrid ageing parameter $\tau= \Omega(n \log n)$  and $\beta\geq 2+\epsilon$ needs 
$O(\tau\cdot n^{3/2})$ fitness function evaluations in expectation to optimise \textsc{Cliff}$_d$ with 
any linear $d \leq n(1/4-c)$ for any arbitrarily small positive constants $c$ and 
$\epsilon$.
\end{theorem}

\begin{proof}
 The process until the cliff point is sampled for the first
time is identical to the previously analysed algorithms with ageing. We will now establish that, given that the parent solution has less than $n-d$ $1$-bits, a mutation of size $n(1-o(1))$ yields an improvement with exponentially small probability. Let $j$ be the number of $0$-bits in the parent solution $x_0$ of the \hypheavy operator and $X$ be the number of $0$-bits that has been flipped to a $1$-bit up to and including the $k$th bit-flip, which is geometrically distributed with expectation $\frac{k\cdot j}{n}$. The number of $0$-bits in the solution sampled after the $k$th bit-flip, $x_k$, is therefore, $j+k-2X$. For the $x_k$ to have a better fitness value than $x_0$, $j+k-2X$ has to be either between $d$ and $d+j$ or smaller than $j$. 

\begin{align*}
j+k-2X &\leq d+j \\
\frac{k-d}{2} &\leq X\\
\frac{k-d}{2}-\frac{k\cdot j}{n}&\leq X-\frac{k\cdot j}{n}\\
\frac{n \cdot (k-d)-2\cdot k\cdot j}{n} &\leq X-E[X]
\end{align*} 
For $n\cdot(k-d)-2\cdot k\cdot j>0$, 
\begin{align*}
&\prob{X-E[X] \geq \frac{n(k-d)-2\cdot k\cdot j}{n}}\\&\leq\exp\left(-\frac{\left(n(k-d)-2\cdot k\cdot j\right)^{2}}{n}\right)
\end{align*}
We will next bound the expression $n\cdot(k-d)-2\cdot k\cdot j$, using our assumptions $k=n(1-o(1))$ and $j<d<\frac{n}{4}-c\cdot n$.

\begin{align*}
&n\cdot\left(k-d\right)-2\cdot k\cdot j>\\&>n\cdot\left(k-\frac{n}{4}+c\cdot n\right)-2\cdot k\cdot \ \left(\frac{n}{4}-c\cdot n\right)\\
&>n\cdot\left(n\left(1-o\left(1\right)\right)-\frac{n}{4}+c\cdot n\right)\\&\quad-2\cdot n\left(1-o\left(1\right)\right)\cdot \ \left(\frac{n}{4}-c\cdot n\right)\\
&>\frac{3n^2}{4}-\frac{n^2}{2}(1-o(1))>\frac{n^2}{4}(1-o(1)).
\end{align*}
Thus, starting from a solution with less than $d$ $0$-bits any mutation of size in the order of $k=n(1-o(1))$ has an exponentially small probability of yielding a solution with better fitness. 

The rest of the proof follows the proof of Theorem~\ref{th:fmutcliff}. Given that $\beta>2+\epsilon$, in the $O(n\log{n})$ generations required to climb the second slope, we never observe a mutation size in the order of $\Omega(n)\setminus n(1-o(1))$ with probability $1-o(1)$. The probability of losing progress while the number of $0$-bits in the cliff solution is in the order of $n(1-o(1))$, \emph{i.e.} when it is close to the edge of the cliff follows the same steps as in the proof of Theorem~\ref{th:fmutcliff} since the probability of improving and the probability of flipping $k<n/2$ are both divided by $2(1-o(1))$ due to the symmetric distribution, which implies that the conditional probability of improving before losing progress stays the same.

 \end{proof}

The performance of the (1+1)~IA$_\beta$ on the other functions analysed in the previous sections is straightforward to bound. For the $\textsc{Partition}$ problem the expected runtime differs by at most a constant factor from that of the \foea if both algorithms use the same $\beta$ parameter. 

\begin{theorem}
Let $S_{n}^{\beta}:=\sum_{i=0}^{n}(\min{(i+1,n-i+1)})^{-\beta}$. Then, the  Fast (1+1)~IA$_\beta$ finds a $(1+\epsilon)$ approximation for any \partition instance in 
$2(S_{n}^{\beta})en^2\cdot(2^{2/\epsilon}+1)+(S_{n}^{\beta})^{-1} (n( 
\epsilon-\epsilon^2))^{\beta}\cdot\epsilon\cdot(\epsilon-\epsilon^2)^{-2/\epsilon}
$ expected fitness function evaluations. (for any $\epsilon=\omega(1/\sqrt{n})$). 
\end{theorem}

\begin{proof}
The proof is identical to that of Theorem~\ref{th:part-doerr} except for the probability of implementing a single bit-flip, which is at least $(S_{n}^{\beta})^{-1}e^{-1}$ for the (1+1)~IA$_\beta$ and the minimum probability of flipping $k$ bits for any $k\in[n]\cup\{0\}$ which is $(S_{n}^{\beta})^{-1}n^{-\beta}$. 
\end{proof}


\begin{table*}[t]
\caption{
{\color{black}
Expected runtimes of the standard (1+1)~EA versus various hypermutation based algorithms. The best asymptotic expected runtimes for each problem are in {\bf bold} font.
\\$^*$: The asymptotic expected runtimes for obtaining a $(1+\epsilon)$ approximation for any constant $\epsilon$.\\ 
$^{**}$: The expected time to find a feasible vertex cover using the node based representation.
\\$^{***}$: The expected time to find a $2$-approximation for vertex cover using the edge representation on a graph with $m$ edges.
\\$^\dagger$ : Holds only for gap sizes in the order of $\Omega(n)$ and at most $n(\frac{1}{4}-\epsilon)$ for some constant $\epsilon>0$. \\
$^{\ddagger}$: The expected runtime is obtained for $dup=1$ and $\gamma=\Omega(1/\log{n})$.\\
$^{\mathsection}$: Optimal runtime obtained when $\gamma=1/(n\log^2{n}$)}. The same expected runtime can be obtained by the $\beta$-algorithms for $\beta=\Omega(\log {n})$.\\
$^{\mathsection\mathsection}$: The expected runtime for the variant of the algorithm which implements hybrid ageing.
} 

 \label{table:afl}
\begin{center}
\begin{tabular}{ |l l l l l H H| }
 \hline
 Function&(1+1)~ EA&\oneoneiahype & \fastianp & Opt-IA &Fast Opt-IA$_\gamma$ &\foea
\T\B \\ \hline
 \textsc{OneMax}& \boldmath{$\Theta(n \log{n})$} \cite{DrosteJansenWegener2002} & $\Theta(n^2 \log{n})$ \cite{CorusOlivetoYazdani2019TCS}& \boldmath{$\Theta\left(n 
\log{n}\left(1+\gamma \log{n}\right)\right)$} &$O\left(\mu \cdot dup \cdot  n^2 \log{n}\right)$ &$O\left(\mu \cdot dup \cdot  n \log{n} (1+\gamma\cdot\log{n}) \right)$ &\textbf{$O(n \log{n})$}
\T\B \\ 
 \textsc{LeadingOnes}& \boldmath{$\Theta(n^{2})$} \cite{DrosteJansenWegener2002} & $\Theta(n^3)$ \cite{CorusOlivetoYazdani2019TCS}& \boldmath{$\Theta\left(n^2 
\left(1+\gamma \log{n}\right)\right)$} &$O\left(\mu \cdot dup \cdot  n^3\right)$ &$O\left(\mu \cdot dup \cdot  n^2 (1+\gamma\cdot\log{n}) \right)$ &$O(n^2)$ \T\B \\
 \textsc{Trap}& ${\Theta(n^n)}$ \cite{DrosteJansenWegener2002}& $\Theta(n^2 \log{n})$ \cite{CorusOlivetoYazdani2019TCS} 
& \boldmath{$\Theta\left(n \log{n}\left(1+\gamma \log{n}\right)\right)$}
 &$O\left(\mu \cdot dup \cdot  n^2 \log{n}\right)$ &$O\left(\mu \cdot dup \cdot n \log{n} (1+\gamma\cdot\log{n})\right)$ &$O(n^{\beta}) $\T\B \\
 $\textsc{Jump}_{d>1}$ 
& ${\Theta(n^d)}$ \cite{DrosteJansenWegener2002}
& $O(n\binom{n}{d})$ \cite{CorusOlivetoYazdani2019TCS} 
&$O\left(\left(\min{\{d,n-d\}}/\gamma\right) \cdot \left(1+\gamma 
\log{n}\right) \cdot \binom{n}{d}\right)$ 
&$O\left(\mu \cdot dup \cdot n \cdot \binom{n}{d}\right)$ &$O\left(\mu \cdot dup \cdot\left(\min{\{d,n-d\}}/\gamma\right) \cdot \left(1+\gamma \log{n}\right) \cdot \binom{n}{d}\right)$ 
&$O(d^{\beta} \binom{n}{d})$\T\B \\
 $\textsc{Cliff}_{d>1}$ & ${\Theta(n^d)}$ \cite{Jorge2015}& $O(n\binom{n}{d})$ \cite{CorusOlivetoYazdani2019TCS} & 
$O\left(\left(d/\gamma\right) \cdot \left(1+\gamma 
\log{n}\right) \cdot \binom{n}{d}\right)$ &$O(n\cdot \binom{n}{d})$ &? &$O(\tau\cdot n^{3/2})$\T\B \\
 $\textsc{HiddenPath}$& $n^{\Omega(n)}$\cite{CorusOlivetoYazdani2019TCS} & $n^{\Omega(\log{n})}$ \cite{CorusOlivetoYazdani2019TCS} & 
$n^{\Omega(\log{n})}$  &$O(\tau\mu n+\mu n^{7/2})^{\ddagger}$\cite{CorusOlivetoYazdani2019TCS} &$O(\tau\mu n+\mu n^{5/2}\log{n})$& $n^{\Omega(\log{n})}$ \cite{}\T\B \\
$\textsc{Partition}^{*}$
& $n^{\Omega(n)}$ \cite{Witt2005}
& $O(n^3)$ \cite{CorusOlivetoYazdaniAIJ2019} 
& \boldmath{$O(n^2 \cdot (1+\gamma \log{n}))$ }& $O(\mu \cdot dup \cdot n^3)$& ?& $O(n^2)$\T\B  \\
$\textsc{Vertex Cover}^{**}$&  \boldmath{$\Theta(n \log{n})$} {\color{black}\cite{JansenOlivetoZargesFOGA2013}} & $\Theta(n^2 \log{n})$ & \boldmath{$\Theta\left(n 
\log{n}\left(1+\gamma \log{n}\right)\right)$} &$O(\mu \cdot dup \cdot n^2 \log{n})$  &? &$\Theta(n \log{n})${\color{black}\cite{}}\T\B  \\
$\textsc{Vertex Cover}^{***}$&  \boldmath{$\Theta(m \log{m})$} \cite{JansenOlivetoZargesFOGA2013} & $\Theta(m^2 \log{m})$ & \boldmath{$\Theta\left(m 
\log{m}\left(1+\gamma \log{m}\right)\right)$} &$O(\mu \cdot dup \cdot m^2 \log{m})$  &? &$\Theta(m \log{m})${\color{black}\cite{}}\T\B  \\
\hline
 \end{tabular}
\end{center}
\end{table*}

\begin{table*}
\begin{center}
\begin{tabular}{ |l H H H H l l l| }
 \hline
 Function&(1+1)~ EA&\oneoneiahype & \fastia & Opt-Ia &Fast Opt-IA$_\gamma$ &\foea & Fast (1+1)~IA$_\beta$
\T\B \\ \hline
 \textsc{OneMax}& $\Theta(n \log{n})$ \cite{DrosteJansenWegener2002} 
& $\Theta(n^2 \log{n})$ \cite{CorusOlivetoYazdani2019TCS}
& $\Theta\left(n 
\log{n}\left(1+\gamma \log{n}\right)\right)$ 
&$O\left(\mu \cdot dup \cdot  n^2 \log{n}\right)$ 
&\boldmath{$O\left(\mu \cdot dup \cdot  n \log{n} (1+\gamma\cdot\log{n}) \right)$ }
&\boldmath{$O(n \log{n})$}
&\boldmath{$O(n \log{n})$}
\T\B \\ 
 \textsc{LeadingOnes}& \boldmath{$\Theta(n^{2})$} \cite{DrosteJansenWegener2002} 
& $\Theta(n^3)$ \cite{CorusOlivetoYazdani2019TCS}
&\boldmath{ $\Theta\left(n^2 \left(1+\gamma \log{n}\right)\right)$}
&$O\left(\mu \cdot dup \cdot  n^3\right)$ 
&\boldmath{$O\left(\mu \cdot dup \cdot  n^2 (1+\gamma\cdot\log{n}) \right)$} 
&\boldmath{$O(n^2)$}&\boldmath{$O(n^2)$} \T\B \\
 \textsc{Trap}
& ${\Theta(n^n)}$ \cite{DrosteJansenWegener2002}
& $\Theta(n^2 \log{n})$ \cite{CorusOlivetoYazdani2019TCS} 
& $\Theta\left(n \log{n}\left(1+\gamma \log{n}\right)\right)$ 
&$O\left(\mu \cdot dup \cdot  n^2 \log{n}\right)$ 
&\boldmath{$O\left(\mu \cdot dup \cdot n \log{n} (1+\gamma\cdot\log{n})\right)$ }
&$O(n^{\beta}) $&\boldmath{$O(n \log{n})$}\T\B \\
 $\textsc{Jump}_{d>1}$ & ${\Theta(n^d)}$ \cite{DrosteJansenWegener2002}& $O(n\binom{n}{d})$ \cite{CorusOlivetoYazdani2019TCS} 
&$O\left(\left(d/\gamma\right) \cdot \left(1+\gamma 
\log{n}\right) \cdot \binom{n}{d}\right)$ &$O\left(\mu \cdot dup \cdot n \cdot \binom{n}{d}\right)$ &$O\left(\mu \cdot dup \cdot\left(\min{\{d,n-d\}}/\gamma\right) \cdot \left(1+\gamma 
\log{n}\right) \cdot \binom{n}{d}\right)$ &$O(d^{\beta} \binom{n}{d})$  &{\color{black}\boldmath{$O\left(\left(\min{\{d,n-d\}}\right)^{\beta} \binom{n}{d}\right)
$}}\T\B \\
 $\textsc{Cliff}_{d>1}$ 
& ${\Theta(n^d)}$ \cite{Jorge2015}
& $O(n\binom{n}{d})$ \cite{CorusOlivetoYazdani2019TCS} 
& $O\left(\left(d/\gamma\right) \cdot \left(1+\gamma 
\log{n}\right) \cdot \binom{n}{d}\right)$ 
&$O(n\cdot \binom{n}{d})$\cite{CorusOlivetoYazdani2019TCS}  
&\boldmath{$O\left(\frac{\mu\cdot dup \cdot \tau\cdot n^2}{d^2}+n\log{n}\right)^{\mathsection}$} 
&$O(\tau\cdot n^{3/2})^{\dagger\mathsection\mathsection}$.
&$O(\tau\cdot n^{3/2})^{\dagger\mathsection\mathsection}$.\T\B \\
 $\textsc{HiddenPath}$& $n^{\Omega(n)}$\cite{CorusOlivetoYazdani2019TCS} & $n^{\Omega(\log{n})}$ \cite{CorusOlivetoYazdani2019TCS} & 
$n^{\Omega(\log{n})}$  &$O(\tau\mu n+\mu n^{7/2})$\cite{CorusOlivetoYazdani2019TCS} &$O(\tau\mu n+\mu n^{5/2}\log{n})^{\ddagger}$& \boldmath{$O(\tau n+ n^{5/2})^{\mathsection\mathsection}$ }&  \boldmath{$O(\tau n+ n^{5/2})^{\mathsection\mathsection}$ }\T\B \\
$\textsc{Partition}^{*}$& $n^{\Omega(n)}$ \cite{Witt2005}& $O(n^3)$ \cite{CorusOlivetoYazdaniAIJ2019} &\boldmath{ $O(n^2 \cdot (1+\gamma \log{n}))$} &$O(\mu \cdot dup \cdot n^3)$&\boldmath{ $O(\mu \cdot dup \cdot n^2 (1+\gamma \log{n}))$}&\boldmath{ $O(n^2)$}
& \boldmath{$O(n^2)$}
\T\B  \\
$\textsc{Vertex Cover}^{**}$&  $\Theta(n \log{n})$ {\color{black}\cite{JansenOlivetoZargesFOGA2013}} & $\Theta(n^2 \log{n})$ & $\Theta\left(n 
\log{n}\left(1+\gamma \log{n}\right)\right)$ &$O(\mu \cdot dup \cdot n^2 \log{n})$ &\boldmath{$\Theta(\mu \cdot dup \cdot n \log{n}(1+\gamma \log{n}))$}&\boldmath{$O(n \log{n})$}&\boldmath{$\Theta(n \log{n})$}\T\B  \\
$\textsc{Vertex Cover}^{***}$&  \boldmath{$\Theta(m \log{m})$} {\color{red}\cite{}} & $\Theta(m^2 \log{m})$ & \boldmath{$\Theta\left(m 
\log{m}\left(1+\gamma \log{m}\right)\right)$} &$O(\mu \cdot dup \cdot m^2 \log{m})$  &\boldmath{$O(\mu \cdot dup \cdot m \log{m}(1+\gamma \log{m}))$} &\boldmath{$\Theta(m \log{m})${\color{black}}}&\boldmath{$\Theta(m \log{m})$}\T\B  \\
\hline
 \end{tabular}
\end{center}
\end{table*}

Thus, the expected runtime of the (1+1)~IA$_\beta$ is in the order of $O(n^2)$ for $1<\beta\leq 2$. 
For $\textsc{OneMax}$ and $\textsc{LeadingOnes}$, its expected runtime asymptotically matches the best possible achievable by unbiased unary randomised search heuristics due to the constant probability of flipping a single bit for any constant $\beta>1$.   
If coupled with ageing, a logarithmic factor may be shaved off from the upper bound on the expected runtime of the the (1+1)~IA$_\beta$ for the $\textsc{HiddenPath}$ function compared to that of the \fastoptia. 
This is due to the higher probability of the (1+1)~IA$_\beta$ of performing 2-bit flips on the slope leading to the hidden path.
The only advantage of \hypfcm over the symmetric power law operator appears for the $\textsc{Cliff}_d$ function which the former can optimise in expected $O(n \log{n})$ fitness evaluations if  is used with ageing, while we could only bound the expected runtime for the (1+1)~IA$_\beta$ by  $O(\tau\cdot n^{3/2})$.
Recall that for the $O(n \log{n})$ bound, a very small $\gamma$ value is required, effectively reducing the hypermutation operator  \hypfcm to perform single bit flips most of the time.
A similar behaviour may be achieved by the (1+1)~IA$_\beta$ by increasing its parameter value to $\beta=\Omega(\log n)$. However the same drawbacks as for the (1+1)~IA$_\gamma$ would be obtained i.e., the algorithm would rarely flip more than one bit.

 Apart for $\textsc{Cliff}_d$, where its upper bound matches that of the \foea, the new symmetric heavy-tailed operator performs asymptotically better, or at least as well as all the alternative operators discussed in this paper while allowing a more robust behaviour for escaping from the local optima of the $\textsc{Jump}$ function compared to the \fastoptia. A summary of the performance of all the considered operators and algorithms is provided in Table~\ref{table:afl}.
 
\section{Conclusion}
Due to recent analyses of increasingly realistic evolutionary algorithms,
higher mutation rates than 
previously recommended, or than those used as a rule of thumb, are gaining significant 
interest in the evolutionary computation community~\cite{OlivetoLehreNeumann2009,Doerretal2017,CorusOlivetoTEVCsteadyGA2017,JumpTEVC2017}.

Such high mutation rates are naturally present in artificial immune systems.
However, previous work has highlighted serious drawbacks of the hypermutation operators traditionally used in the AIS field.
Firstly, while they allow to escape from local optima faster than the standard bit mutations (SBM) used by evolutionary algorithms, they do so at the expense of often being a linear factor slower at hillclimbing in the exploitation phases of the search~\cite{CorusOlivetoYazdani2019TCS,JansenZargesTCS2011,JansenOlivetoZarges2011vertex}. Secondly, the `hypermutations with mutation potential' (HMP) operators used in Opt-IA cancel out the power of the ageing operator to escape from local optima by accepting solutions of lower quality.
We have presented an alternative HMP operator, \hypfcm,  that provably removes these drawbacks and we have rigorously shown, 
for several significant benchmark problems from the literature and for classical problems from combinatorial optimisation, that it
maintains the exploration characteristics of the traditional operators while outperforming them by up to linear factor speed-ups in the exploitation phases.
These speed-ups at hillclimbing allow them to quickly provide feasible solutions, and high quality approximations for the NP-Hard \partition and \textsc{Vertex cover} problems a linear factor faster than the HMP operators traditionally used in the literature. A careful comparison with other {\it fast} mutation operators from the literature confirms the validity of our proposed {\it fast} hypermutation operators.

The main modification that allows to achieve the presented improvements over the standard static HMP with FCM is to sample the solution after the $i$th bit-flip stochastically 
rather than deterministically
with probability one. 
Importantly, by using a symmetric power-law distribution, we have also shown how it is possible to avoid using the FCM mechanism altogether and just evaluate one search point per hypermutation. This was probably the originally desired behaviour for the hypermutation operator of Opt-IA. However, the standard static HMP is inefficient for any function with up to a polynomial number of optima without the use of FCM~\cite{CorusOlivetoYazdani2019TCS}.
Furthermore, the power-law distribution allows the {\it fast} HMP operators to work in harmony with ageing to escape from local optima by accepting solutions of inferior quality. This behaviour was not possible with the original static HMP, thus considerably limiting the power of the Opt-IA algorithm where both operators are employed.

We point out that while the presented operators naturally fit within AISs, there is no reason to believe that they should not also be effective if employed within any randomised search heuristic, including EAs.

Since the optimal values for the distribution parameters $\gamma$ and $\beta$ are different in the exploitation  and the exploration phases,
future work may consider an adaptation of the parameters 
to automatically allow them to increase and decrease throughout the run~\cite{DLOW2018,DoerrDoerr2018,WarwickerAAAI20}.
Furthermore, the performance of the proposed operators should be evaluated experimentally for classical combinatorial optimisation problems, complementing the theoretical analyses of the worst-case performance, and for real-world applications.

\bibliographystyle{unsrt}
\bibliography{mybib2} 

\ifCLASSOPTIONcaptionsoff
  \newpage
\fi

\begin{IEEEbiography}{Michael Shell}
Biography text here.
\end{IEEEbiography}
%
\begin{IEEEbiographynophoto}{John Doe}
Biography text here.
\end{IEEEbiographynophoto}
%
%
\begin{IEEEbiographynophoto}{Jane Doe}
Biography text here.
\end{IEEEbiographynophoto}




\end{document}